\documentclass[11pt]{article}

\usepackage[final]{acl}

\usepackage{times}
\usepackage{latexsym}

\usepackage[T1]{fontenc}

\usepackage[utf8]{inputenc}

\usepackage{microtype}

\usepackage{inconsolata}

\usepackage{graphicx}
\usepackage{hyperref}
\usepackage{url}
\usepackage{algorithm}
\usepackage{algorithmic}
\usepackage{subcaption}
\usepackage{multirow}
\usepackage{booktabs}
\usepackage[table]{xcolor}
\usepackage{colortbl}
\usepackage{amsmath}
\usepackage{amsthm}
\usepackage{amsfonts}
\newtheorem{theorem}{Theorem}
\newtheorem{lemma}{Lemma}
\usepackage{graphicx}
\usepackage{enumitem}
\setlist[itemize]{leftmargin=*}

\title{Toward Robust In-Context Learning: Leveraging Out-of-distribution Proxies for Target Inaccessible Demonstration Retrieval}

\author{
Hao Xu$^{1}$, Rite Bo$^{1}$, Fausto Giunchiglia$^{1,2}$, Yingji Li$^{1}$, Rui Song$^{1}$\thanks{Corresponding author.} \\
$^{1}$College of Computer Science and Technology, Jilin University, China \\
$^{2}$Department of Information Engineering and Computer Science, University of Trento, Italy \\
\texttt{
\{xuhao,yingjili,songrui\}@jlu.edu.cn,} \\
\texttt{
bort24@mails.jlu.edu.cn, fausto.giunchiglia@unitn.it
}
}

\begin{document}
\maketitle
\begin{abstract}
Although studies have demonstrated that Large Language Models (LLMs) can perform well on Out-of-Distribution (OOD) tasks, their advantage tends to diminish as the distribution shift becomes more severe. Consequently, researchers aim to retrieve distributionally similar and informative demonstrations from the available source domain to boost the inference capabilities of LLMs. However, in practical scenarios where the target domain is inaccessible, evaluating the unknown distribution is challenging, which indirectly impacts the quality of the selected demonstrations. To address this problem, we propose \textbf{DOPA}, a demonstration search framework that incorporates an OOD proxy to approximate the inaccessible target domain and guide the retrieval process. Building on proxy-based evaluation, DOPA further introduces a Mahalanobis distance-based global diversity constraint to ensure sufficient diversity among the retrieved demonstrations. Experimental results on multiple LLMs and tasks demonstrate that DOPA effectively enhances robustness in OOD settings\footnote{https://github.com/bort64/ood\_code}.
\end{abstract}

\section{Introduction}
Large language models (LLMs) have achieved strong performance across diverse NLP tasks~\cite{ChangWWWYZCYWWYZCYYX24, SongLTWGX25}, with in-context learning (ICL) emerging as a widely used prompting paradigm~\cite{MinLHALHZ22}. By providing a small set of demonstrations, ICL can effectively guide model reasoning and prediction. However, recent studies show that LLM performance degrades markedly under out-of-distribution (OOD) settings~\cite{YuanCCGZCJL023, 00020Y025}, particularly when demonstrations are distributionally mismatched with the target domain, motivating research into more robust demonstration selection strategies.

Retrieval~\cite{Man2024Survey} and augmentation~\cite{0004LHZLTCM24} are two commonly used approaches for obtaining effective samples. The former searches for the most relevant examples within a specific domain, while the latter rewrites existing samples to reduce their discrepancy with the target instance. Demonstration retrieval relies on a retriever. Some off-the-shelf metrics, such as Bm25~\cite{AgrawalZLZG23}, sentence encoder-based similarity~\cite{LiuSZDCC22}, model influence~\cite{PengDY00OT24, Vinay2024}, and misconfidence~\cite{Shangqing2024}, can support general-purpose retrieval strategies. Meanwhile, other approaches aim to train a dense retriever to obtain more task-relevant retrieval results~\cite{ChengHBZLW0WDZ23, LiLYLZNXWQ23}. Augmentation, on the other hand, focuses on adapting existing samples to better match the distributional characteristics of the target instance~\cite{Kyle2024, Madine24}. \textit{\textbf{However}, in real-world applications, an inaccessible target domain hinders the ability to obtain domain-aligned demonstrations, often resulting in degraded performance~\cite{0008GLTX24}.}

To address the aforementioned challenge, we propose a \textbf{d}emonstration \textbf{o}ptimization framework based on OOD \textbf{p}roxy \textbf{a}ssessment (termed \textbf{DOPA}). This framework quantifies the utility of source-domain samples in the absence of target-domain access, and leverages the quantification results to guide demonstration retrieval. At its core, DOPA introduces an OOD proxy as a principled approximation to the unknown target distribution~\cite{zhang2022falsehoods}, which is composed of two components: a source proxy and a target proxy. The source proxy is defined as an instruction-tuned LLM trained on the source domain to fully adapt to the source distribution, while the target proxy corresponds to the original, unmodified version of the same LLM. The perplexity ratio between their predictions on identical input samples is adopted as the OOD score for those samples~\cite{NalisnickMTGL19}. This OOD score serves to estimate the degree of familiarity of source-domain samples with the target domain in the absence of target-domain information. It is further integrated with representational similarity to predicted samples for candidate selection. The validity of the OOD score is theoretically supported through a bounded proxy error analysis. Moreover, to enhance the diversity of retrieved demonstrations, we incorporate a Mahalanobis distance-based search strategy into the retrieval process. By relying on the OOD proxy, DOPA is capable of identifying informative demonstrations solely within the source domain, without requiring any samples from the target domain. Extensive experiments show that DOPA consistently outperforms baseline approaches across diverse LLMs and natural language understanding tasks. In addition, we provide a multi-dimensional analysis that demonstrates the effectiveness of the proxy in selecting samples that exhibit behavioral similarity to those in the target domain. Our contributions are as follows:

(i) We propose a method that leverages OOD proxies to extract distribution-aligned samples, and we theoretically demonstrate the soundness of the proxy through a bounded proxy error guarantee. (ii) We propose a target-agnostic demonstration retrieval framework based on OOD proxies, which combines proxy results and contextual diversity to enhance the quality of demonstration selection. (iii) Experimental results on multiple NLP tasks across various LLMs demonstrate that DOPA effectively enhances OOD robustness in ICL.

\section{Related Work}

\textbf{Demonstration Retrieval}. Despite the impressive performance demonstrated by ICL, an increasing number of studies have shown its sensitivity to the choice of demonstrations~\cite{Shortcut2024}. To obtain more effective demonstrations, a natural idea is to search over candidate samples within a constrained space~\cite{Man2024Survey}. Depending on whether the retrieval tool has been trained, demonstration search can be divided into off-the-shelf retrieval and retrieval based on fine-tuned models. Term-based similarity has been widely used for demonstration retrieval, with BM25 being one of the most popular scoring metrics~\cite{AgrawalZLZG23, Ye0F0K23}. In addition, several sentence embedding models, such as SBERT~\cite{WangYW24}, RoBERTa~\cite{LiuSZDCC22}, and SimCSE~\cite{GaoYC21}, have also been widely used to compute inter-sample similarity and optimize demonstration selection. Moreover, some approaches assess the influence of individual samples on model predictions to select high-impact examples for demonstrations~\cite{PengDY00OT24, Vinay2024}. Off-the-shelf retrieval methods may yield suboptimal results, as they do not incorporate task-specific information. Therefore, some methods have explored leveraging feedback signals from LLMs to distinguish between important and unimportant samples, and further optimize the retriever for specific tasks using objectives such as ranking~\cite{LiLYLZNXWQ23}, contrastive learning~\cite{ChengHBZLW0WDZ23, Man2023}, and diversity~\cite{Ye0F0K23}. But these methods often rely on feedback from LLMs, which leads to higher computational complexity.

\textbf{OOD Robustness in ICL}. In ICL settings, distribution shifts often cause substantial performance degradation, exposing models’ sensitivity and limited robustness to unseen domains~\cite{YuanCCGZCJL023, 00020Y025}. Such distributional gaps can undermine demonstration retrieval, as retrieved examples may no longer align semantically with the target task. To address this issue, prior work has explored demonstration augmentation, including incorporating external knowledge such as linguistic rules~\cite{JiangC0PL24} or human feedback~\cite{Haoyue2024} for fine-tuning LLMs. However, the necessity of fine-tuning remains debated, with some studies suggesting that LLMs inherently possess the capacity to handle OOD data~\cite{UppaalHL23, Andi2024}. Motivated by this perspective, semantic rewriting has emerged as an alternative, prompting LLMs to adapt source samples to better align with the target domain~\cite{Kyle2024, Madine24}.

\section{Method}
\begin{figure*}[t]
    \centering
    \includegraphics[width=0.95\linewidth]{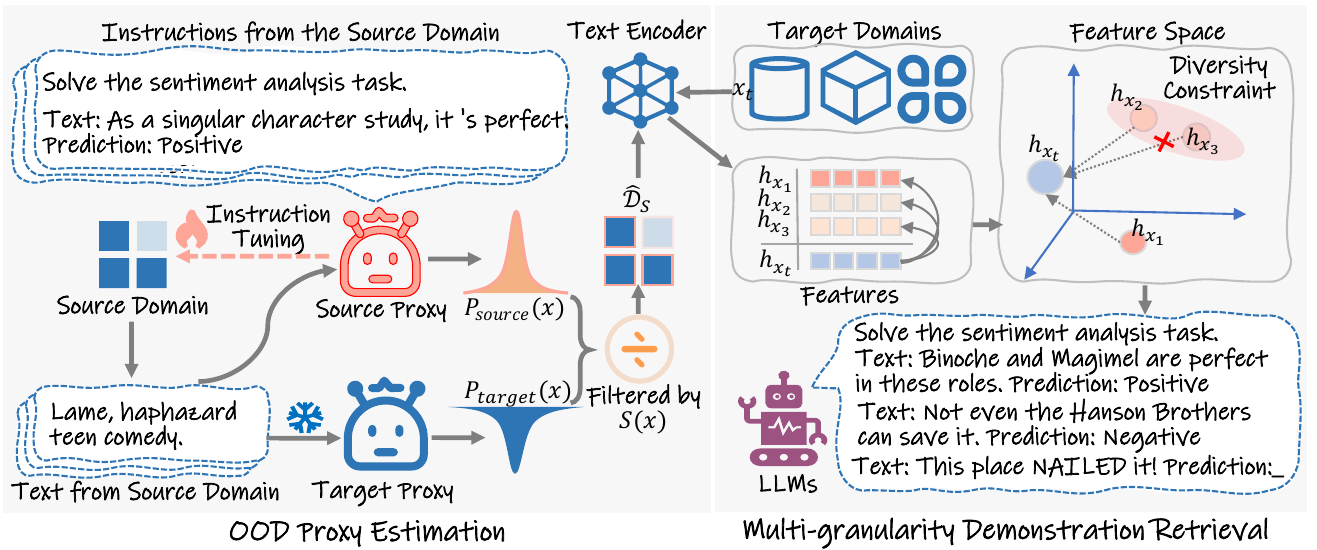}
    \caption{The model architecture of DOPA based on the sentiment analysis task. First, DOPA performs task-specific instruction tuning on the source domain to obtain a source proxy based on any given LLM. Correspondingly, an identical LLM without fine-tuning, which preserves the prior knowledge of the target domain, is employed as the target-domain proxy. For the same input, the ratio between the two proxies is employed as an OOD proxy estimation, which is further combined with similarity and diversity to support multi-granularity demonstration search.}
    \label{fig:model}
\end{figure*}

\subsection{Task Definitions and Model Description}
Our model description begins with some definitions. In the OOD setting, LLMs $\mathcal{M}$ are restricted to using data from $\mathcal{D}_S$ to perform ICL, and are expected to make predictions on any sample $x_t$ from $\mathcal{D}_T$ as accurately as possible. During the inference process of LLMs, all samples from $\mathcal{D}_T$ other than $x_t$ are strictly inaccessible, preventing the model from making decisions by referencing samples from a similar distribution. For ICL, a prompt $\mathcal{P}{x_t}$ is constructed by selecting $N\times |Y|$ labeled examples ${(x^{(j)}, y^{(j)})}_{j=1}^{N\times |Y|}$ from $\mathcal{D}_S$, which are then concatenated with $x_t$ and fed into any $\mathcal{M}$. Here, $|Y|$ denotes the size of the label space. Then, the LLM produces a prediction $\hat{y}_t = \mathcal{M}(\mathcal{P}{x_t})$. In different task settings, $\hat{y}_t$ can take various forms depending on the output space. For classification tasks, it typically corresponds to a token representing a label category (e.g., positive or negative), while for generative tasks, it may be a string representing the desired output. As illustrated in Figure~\ref{fig:model}, DOPA comprises two main components: OOD proxy estimation and multi-granularity demonstration retrieval. The proxy estimation module assesses the proximity of source domain samples to the target domain using an OOD proxy, while the demonstration retrieval module selects appropriate examples by jointly optimizing semantic similarity and diversity constraints. 

\subsection{OOD Proxy Estimation}
The goal of the OOD proxy estimation is to evaluate the utility of source domain samples to select those that are more aligned with the target domain. But without access to the target domain, it is difficult to accurately assess the target distribution. Therefore, inspired by prior work on OOD detection~\cite{RenLFSPDDL19, zhang2022falsehoods}, we construct an OOD proxy to approximate the target domain distribution, and compute the OOD score of any sample via the proxy. The OOD score is then used to guide sample selection from $\mathcal{D}_S$.

\textbf{Proxy Construction}. 
The OOD proxy consists of two components: the source proxy and the target proxy, which ideally model the source and target distributions, respectively. For the former, an intuitive approach is to instruction-tune LLMs on the source domain so that the model can better adapt to the source distribution. In DOPA, the instructions for source proxy are encapsulated in the same format as in ICL, aiming to prompt LLMs to produce reasonable task-related predictions. Formally, for easily accessible source domain data $x_s \in \mathcal{D}_S$, DOPA optimizes the LLM through the following cross-entropy supervised loss:
\begin{equation}
    \mathcal{L}_{sft} = -\sum_{j=1}^T\text{log} p_{\mathcal{M}}(y_{i,j}|\mathcal{P}_{x_s}, y_{i,\textless t}),
\end{equation}
where $\mathcal{P}_{x_s}$ is a task-related prompt that does not contain any demonstrations. 

As for the target domain, since the target domain distribution is unknown, some methods propose a general approach by replacing the target-domain proxy with a uniform distribution~\cite{bishop1993novelty, NalisnickMTGL19}. Such a strong assumption is inherently destined to yield suboptimal results as shown in Lemma~\ref{lem:1}. Given that LLMs are pretrained on extensive corpora, it is reasonable to assume that they implicitly encode a broad spectrum of linguistic and factual knowledge. As such, LLMs can act as weak proxies for the target distribution, particularly in few-shot settings~\cite{Andi2024}. 

\textbf{Sample Screening based on OOD Score}.
Given the aforementioned proxies, DOPA further assumes that if a sample exhibits divergent behavior under these two proxies, it may suggest an inherent bias or a stronger alignment toward one specific domain. This facilitates domain discrimination in the absence of any auxiliary target domain samples. In the previous research, the likelihood ratio is one of the most commonly used detection criteria for the divergent behavior~\cite{RenLFSPDDL19, zhang2022falsehoods}:
\begin{equation}
    S(x) = \frac{P_{target}(x)}{P_{source}(x)} \approx \underbrace{\frac{P_{\text{target}}^{\text{proxy}}(x)}{P_{\text{source}}^{\text{proxy}}(x)}}_{\text{OOD proxy}},
\end{equation}
where $P_{source}(x)$ and $P_{target}(x)$ represent the behavior of models with the source and target domain distributions when given the same input sample $x$, respectively. Under distributional uncertainty, the OOD proxy is used as their approximation. To support the validity of OOD proxy estimation, we further establish a theoretical guarantee on the boundedness of proxy error under mild assumptions.

\begin{theorem}[\underline{Proxy Error Bound}]\label{th:1}
    Let $P_{\mathrm{target}}$ and $P_{\mathrm{source}}$ be the true probability distributions of the target and source domain, let $P_{\mathrm{target}}^{\mathrm{proxy}}$ and $P_{\mathrm{source}}^{\mathrm{proxy}}$ be the corresponding proxy distributions. Suppose there exist constants $\varepsilon_t \geq 0$, $\varepsilon_s \geq 0$, $m_t > 0$, and $m_s > 0$ such that the following hold:
    \begin{itemize}
    \setlength{\leftmargin}{0pt}
        \item The Kullback-Leibler divergences are bounded: $
        D_{\mathrm{KL}}(P_{\mathrm{target}} \parallel P_{\mathrm{target}}^{\mathrm{proxy}}) \leq \varepsilon_t, 
        D_{\mathrm{KL}}(P_{\mathrm{source}} \parallel P_{\mathrm{source}}^{\mathrm{proxy}}) \leq \varepsilon_s$. The proxy distributions have pointwise lower bounds: $
        P_{\mathrm{target}}(x), P_{\mathrm{target}}^{\mathrm{proxy}}(x) \geq m_t, \quad
        P_{\mathrm{source}}(x), P_{\mathrm{source}}^{\mathrm{proxy}}(x) \geq m_s. $
    \end{itemize}
    Then, for all $x$, the error in the log-likelihood ratio satisfies:
    \[
    \left| \log \frac{P_{\mathrm{target}}(x)}{P_{\mathrm{source}}(x)} - \log \frac{P_{\mathrm{target}}^{\mathrm{proxy}}(x)}{P_{\mathrm{source}}^{\mathrm{proxy}}(x)} \right| \leq \frac{\varepsilon_t}{m_t} + \frac{\varepsilon_s}{m_s}.
    \]
\end{theorem}

\begin{lemma}[\underline{Error Bound with Uniform Proxy}] \label{lem:1}
    Building upon Theorem~\ref{th:1}, if a uniform distribution is used as the proxy for target, it is more likely to result in a looser upper bound on the error.
\end{lemma}

Due to space limitations, the proof of the above theorem is provided in Appendix~\ref{app:th}. Such a uniform bound certifies that the proxy‐based score deviates from the true likelihood ratio by at most a known quantity, thereby providing theoretical assurance for reliable estimation. In the case of autoregressive LLMs, perplexity~\cite{wuhrmann2025low} is commonly employed to quantify the model’s familiarity with a given text $x$: $PPL(x)=exp \big(-\frac{1}{m} \sum_{i=1}^m logP(w_i|w_{<i}) \big)$, where $m$ is the total number of tokens in $x$, and $P(w_i|w_{<i})$ is the conditional probability of the language model predicting the $i$-th token. Therefore, to conform to the $log$ form as stated in Theorem~\ref{th:1}, we adopt the log-perplexity difference as a more stable alternative:
\begin{equation}
    S(x) = logPPL_{target}^{proxy}(x) - logPPL_{source}^{proxy}(x). 
\end{equation}
Ideally, if the value of $S(x)$ is relatively low, it exhibits higher perplexity under the source-domain proxy and lower perplexity under the target-domain proxy, which further indicates that the sample is more aligned with the target domain and should therefore be prioritized for constructing demonstrations. By performing a single pass over $\mathcal{D}_S$, we can obtain a potential subset $\hat{\mathcal{D}}_S$ that is closer to the target domain distribution by selecting the $k$ samples with the lowest OOD scores. 

\subsection{Demonstration Retrieval}
Although the OOD scores help identify source domain samples that are more likely to align with the target domain, the resulting coarse-grained subset still requires further refinement to construct effective demonstrations. Existing studies have provided strong support for the demonstration search process~\cite{LiuSZDCC22, AgrawalZLZG23}, a general approach is to adopt an off-the-shelf text representation model to encode candidate texts into vectors and rank the most relevant demonstrations based on their cosine similarity with the test sample. But one limitation of proxy-based OOD scoring lies in its reliance on language model perplexity, which primarily captures token-level fluency and distributional similarity. As a result, it may implicitly favor shorter texts or those conforming to high-frequency linguistic patterns~\cite{HoltzmanBDFC20}, leading to reduced diversity in the selected sample pool and potentially impairing the quality of the retrieved demonstrations. To address this issue, we further introduce a global diversity constraint to improve the overall quality of the retrieved demonstrations. Specifically, for each sample representation $h_{x_i}$ corresponding to the proxy-filtered set $\hat{\mathcal{D}}_S$, we initialize a candidate sample set $\mathcal{D}_{demo}$ based on the similarity of the representations to $h_t$:
\begin{equation}
    \label{eq:init}
    \mathcal{D}_{demo} = arg \max_{C} \{sim(h_{x_i}, h_t)\}_{i=1}^{|\hat{\mathcal{D}}_S|},
\end{equation}
where $C$ is the number of samples in the initialized candidates\footnote{We specify the value of $C$ in the detailed experimental settings.}. Subsequently, the mean pairwise Mahalanobis distance~\cite{LiDWW23} among samples in $\mathcal{D}_{demo}$ is used to quantify the diversity:
\begin{equation}
    Div = \frac{2}{|\mathcal{D}_{demo}|(|\mathcal{D}_{demo}|-1)}\sum_{i<j}\sqrt{\mathbf{D}^{\top}_{ij} \Sigma^{-1} \mathbf{D}_{ij}},
\end{equation}
where $\mathbf{D}_{ij}=h_{x_i} - h_{x_j}$, $\Sigma$ is the empirical covariance matrix computed over all samples. The Mahalanobis distance is adopted because it accounts for the correlations between samples while measuring diversity, which helps impose constraints on similarity-based retrieval results. If a new sample $\hat{x} \in \{\hat{\mathcal{D}}_S-\mathcal{D}_{demo}$\} does not lead to a decrease in overall diversity i.e. $Div_{\mathcal{D}_{demo}} \leq Div_{\{\hat{x}\} \cup \mathcal{D}_{demo}}$, it is retained. This process continues until the number of samples meets the required threshold for constructing demonstrations. The final selected demonstration set is used for ICL. The above procedure is summarized in Algorithm~\ref{alg:algorithm}. After obtaining sufficient samples, we construct demonstrations in a fixed label order to prevent bias introduced by orders, and use them for ICL.

\begin{algorithm}[tb]
\label{algo}
\caption{Demonstration Retrieval Process of DOPA}
\label{alg:algorithm}
\textbf{Input}: Proxy-filtered set $\hat{\mathcal{D}_S}$, test sample $x_t$.\\
\textbf{Parameter}: Demonstration quantity $N\times |Y|$, initialized candidate set $C$. \\
\textbf{Output}: Final demonstration set $\mathcal{D}_{demo}$ with size $N$. 
    \begin{algorithmic}[1] 
        \STATE Init $\mathcal{D}_{demo}$ by Eq.\ref{eq:init}, sort $\hat{\mathcal{D}_S}$ in ascending order according to $sim(h_{x_i}, h_t)$, counter$\gets 0$. 
        \WHILE{$|\mathcal{D}_{demo}| < N\times |Y|$}
            \STATE $\hat{x} \gets \hat{\mathcal{D}}_S[C+counter]$.
            \IF {$Div_{\mathcal{D}_{demo}} \leq Div_{\{\hat{x}\} \cup \mathcal{D}_{demo}}$}
            \STATE $\mathcal{D}_{demo}\gets\{\hat{x}\} \cup \mathcal{D}_{demo}$.
            \ENDIF
            \STATE counter$\gets$counter + 1. 
        \ENDWHILE
        \STATE \textbf{return} $\mathcal{D}_{demo}$
\end{algorithmic}
\end{algorithm}

\definecolor{lightgray}{gray}{0.9}
\definecolor{rowgray}{gray}{0.88}
\definecolor{colblue}{rgb}{0.85,0.92,1}
\begin{table*}[h]
  \setlength\tabcolsep{4.8pt}
  \renewcommand{\arraystretch}{0.3}
  \centering
  \footnotesize
    \begin{tabular}{c!{\color{white}\vrule width 1pt}
    c!{\color{white}\vrule width 1pt}c!{\color{white}\vrule width 1pt}c!{\color{white}\vrule width 1pt}>{\columncolor{colblue}}c!{\color{white}\vrule width 1pt}
    c!{\color{white}\vrule width 1pt}c!{\color{white}\vrule width 1pt}c!{\color{white}\vrule width 1pt}>{\columncolor{colblue}}c!{\color{white}\vrule width 1pt}
    c!{\color{white}\vrule width 1pt}c!{\color{white}\vrule width 1pt}c!{\color{white}\vrule width 1pt}>{\columncolor{colblue}}cc}
    \toprule
     \multirow{2}{*}{\textbf{Methods}} 
      & \multicolumn{4}{c}{\textbf{SA}} 
      & \multicolumn{4}{c}{\textbf{TD}} 
      & \multicolumn{4}{c}{\textbf{NLI}} \\ 
      \cmidrule(lr){2-5} \cmidrule(lr){6-9} \cmidrule(lr){10-13}
     & dynasent & semeval & sst & \textbf{avg} 
                  & implicit & adv & toxigen & \textbf{avg} 
                  & wanli & anli & cnli & \textbf{avg} \\ 
    \midrule
    \rowcolor{black!8} \multicolumn{13}{c}{\textbf{\texttt{GPT2-xl}}} \\
          Random & 36.33 & 49.28 & 47.70 & 44.44 & 50.47 & 50.20 & 50.60 & 50.42 & 34.23 & 32.23 & 39.22 & 35.23 \\
          KNN   & 35.89 & 45.92 & 51.17 & 44.33 & 47.67 & 47.50 & 48.50 & 47.89 & 33.57 & 33.50 & 45.05 & 37.37 \\
          DrICL & 37.00 & 47.66 & 52.76 & 45.81 & 49.70 & 51.38 & 48.23 & 49.77 & 32.03 & 33.33 & 47.20 & 37.52 \\
          Rewrite & 36.00 & 45.84 & 50.61 & 44.15 & 46.90 & 45.72 & 49.37 & 47.33 & 34.00 & 32.77 & 45.38 & 37.38 \\
          InfICL & 36.61 & 49.26 & 46.95 & 44.27 & 49.90 & 50.33 & 50.17 & 50.13 & 33.80 & 32.40 & 24.25 & 30.15 \\
          DICL & 38.46 & 49.86 & 56.92 & 48.41 & 46.20 & 47.37 & 46.37 & 46.65 & 32.93 & 33.47 & 44.57 & 36.99 \\
          \cellcolor{rowgray}DOPA*   & \cellcolor{rowgray}38.23 & \cellcolor{rowgray}48.44 & \cellcolor{rowgray}59.61 & \textbf{48.76} & \cellcolor{rowgray}51.67 & \cellcolor{rowgray}53.29 & \cellcolor{rowgray}50.50 & \textbf{51.82} & \cellcolor{rowgray}34.93 & \cellcolor{rowgray}33.43 & \cellcolor{rowgray}45.77 & \textbf{38.04} \\
    \midrule
    \rowcolor{black!8} \multicolumn{13}{c}{\textbf{\texttt{LLaMA3.2-3B}}} \\
          Random & 53.81 & 47.86 & 66.26 & 55.98 & 57.70 & 55.20 & 65.70 & 59.53 & 37.70 & 35.50 & 42.75 & 38.65 \\
          KNN   & 52.63 & 45.76 & 65.42 & 54.60 & 56.63 & 53.29 & 51.03 & 53.65 & 37.20 & 34.67 & 44.24 & 38.70 \\
          DrICL & 56.05 & 46.08 & 67.10 & 56.41 & 57.83 & 56.18 & 64.80 & 59.61 & 36.50 & 33.87 & 42.95 & 37.77 \\
          Rewrite & 53.92 & 45.06 & 64.29 & 54.43 & 51.57 & 57.04 & 62.70 & 57.10 & 36.43 & 35.60 & 42.75 & 38.26 \\
          InfICL & 53.35 & 46.80 & 64.39 & 54.84 & 56.33 & 55.20 & 65.57 & 59.03 & 36.23 & 36.00 & 40.65 & 37.63 \\
          DICL & 53.79 & 47.50 & 68.42 & 56.57 & 56.60 & 53.88 & 66.00 & 58.83 & 37.67 & 34.60 & 42.52 & 38.26 \\
          \cellcolor{rowgray}DOPA*   & \cellcolor{rowgray}55.71 & \cellcolor{rowgray}53.28 & \cellcolor{rowgray}68.88 & \textbf{59.29} & \cellcolor{rowgray}57.87 & \cellcolor{rowgray}56.45 & \cellcolor{rowgray}65.30 & \textbf{59.87} & \cellcolor{rowgray}38.40 & \cellcolor{rowgray}35.87 & \cellcolor{rowgray}43.19 & \textbf{39.15} \\
    \midrule
    \rowcolor{black!8} \multicolumn{13}{c}{\textbf{\texttt{Gemma2-2B}}} \\
          Random & 56.47 & 47.06 & 66.45 & 56.66 & 55.57 & 56.51 & 63.93 & 58.67 & 33.37 & 33.00 & 42.66 & 36.34 \\
          KNN   & 55.29 & 47.28 & 66.26 & 56.28 & 53.20 & 47.89 & 63.87 & 54.99 & 33.50 & 32.93 & 41.61 & 36.01 \\
          DrICL & 57.67 & 47.20 & 67.10 & 57.32 & 55.43 & 56.45 & 61.17 & 57.68 & 33.73 & 33.57 & 45.43 & \textbf{37.58} \\
          Rewrite & 57.91 & 47.12 & 67.01 & 57.35 & 48.57 & 51.84 & 61.84 & 54.08 & 33.70 & 33.33 & 45.29 & 37.44 \\
          InfICL & 58.07 & 45.38 & 64.57 & 56.01 & 55.63 & 57.50 & 59.90 & 57.68 & 33.27 & 32.93 & 45.29 & 37.16 \\
          DICL & 55.61 & 48.66 & 68.10 & 57.46 & 54.40 & 53.68 & 65.30 & 57.79 & 33.43 & 33.33 & 45.15 & 37.30 \\
          \cellcolor{rowgray}DOPA*   & \cellcolor{rowgray}57.24 & \cellcolor{rowgray}47.70 & \cellcolor{rowgray}68.13 & \textbf{57.69} & \cellcolor{rowgray}56.53 & \cellcolor{rowgray}58.09 & \cellcolor{rowgray}65.73 & \textbf{60.12} & \cellcolor{rowgray}33.37 & \cellcolor{rowgray}33.07 & \cellcolor{rowgray}46.10 & 37.51 \\
    \midrule
        \rowcolor{black!8} \multicolumn{13}{c}{\textbf{\texttt{Qwen3-1.7B}}} \\
          Random & 62.82 & 60.90 & 69.17 & 64.29 & 54.97 & 52.89 & 67.20 & 58.35 & 41.30 & 35.17 & 39.12 & 38.53 \\
          KNN & 60.75 & 58.32 & 70.67 & 63.25 & 56.10 & 50.13 & 65.73 & 57.32 & 41.77 & 35.20 & 37.83 & 38.27 \\
          DrICL & 61.54 & 60.16 & 70.38 & 64.03 & 54.07 & 55.86 & 66.37 & 58.76 & 42.77 & 35.33 & 37.49 & 38.53 \\
          Rewrite    & 60.38 & 54.72 & 70.29 & 61.80 & 51.33 & 57.50 & 61.97 & 56.93 & 39.47 & 36.63 & 38.79 & 38.30 \\
          InfICL    & 62.10 & 61.12 & 69.92 & 64.38 & 55.30 & 56.83 & 65.23 & 59.12 & 40.80 & 35.57 & 40.03 & 38.80 \\
          DICL & 59.94 & 57.00 & 69.82 & 62.26 & 55.37 & 54.54 & 64.45 & 58.12 & 42.70 & 35.83 & 39.22 & 39.25 \\
          \cellcolor{rowgray}DOPA*   & \cellcolor{rowgray}63.35 & \cellcolor{rowgray}59.64 & \cellcolor{rowgray}71.79 & \textbf{64.93} & \cellcolor{rowgray}55.47 & \cellcolor{rowgray}56.45 & \cellcolor{rowgray}65.73 & \textbf{59.22} & \cellcolor{rowgray}42.37 & \cellcolor{rowgray}36.47 & \cellcolor{rowgray}40.94 & \textbf{39.93} \\
    \midrule
    \rowcolor{black!8} \multicolumn{13}{c}{\textbf{\texttt{LLaMA3.1-8B}}} \\
          Random & 58.46 & 52.02 & 69.63 & 60.04 & 56.50 & 55.66 & 66.40 & 59.52 & 39.73 & 36.67 & 43.09 & 39.83\\
          KNN & 57.61 & 49.32 & 70.48 & 59.13 & 57.23 & 54.80 & 65.50 & 59.18 & 41.00 & 37.40 & 42.23 & 40.21\\
          DrICL & 59.20 & 51.22 & 69.73 & 60.05 & 58.57 & 56.64 & 67.73 & 60.98 & 40.70 & 35.90 & 42.32 & 39.64 \\
          Rewrite & 56.59 & 48.74 & 69.07 & 58.13 & 54.57 & 59.67 & 63.87 & 59.37 & 38.87 & 36.00 & 43.23 & 39.37 \\
          InfICL & 57.63 & 52.92 & 70.76 & 60.44 & 59.00 & 59.74 & 63.77 & 60.83 & 40.87 & 37.23 & 43.04 & 40.38 \\
          DICL & 58.35 & 51.44 & 70.10 & 59.96 & 57.03 & 55.99 & 65.97 & 59.66 & 38.97 & 36.37 & 42.13 & 39.16 \\
          \cellcolor{rowgray}DOPA*   & \cellcolor{rowgray}59.25 & \cellcolor{rowgray}51.84 & \cellcolor{rowgray}72.16 & \textbf{61.08} & \cellcolor{rowgray}58.60 & \cellcolor{rowgray}59.28 & \cellcolor{rowgray}67.43 & \textbf{61.77} & \cellcolor{rowgray}41.33 & \cellcolor{rowgray}37.53 & \cellcolor{rowgray}42.75 & \textbf{40.54} \\
    \midrule
    \rowcolor{black!8} \multicolumn{13}{c}{\textbf{\texttt{Qwen3-8B}}} \\
          Random & 68.56 & 63.42 & 76.85 & 69.61 & 56.90 & 56.71 & 65.63 & 59.75 & 41.13 & 35.77 & 28.17 & 35.02\\
          KNN & 65.34 & 63.78 & 76.01 & 68.38 & 57.17 & 58.03 & 68.07 & 61.09 & 41.93 & 34.93 & 30.27 & 35.71\\
          DrICL & 68.40 & 63.82 & 78.26 & 70.16 & 63.10 & 57.04 & 77.43 & 65.86 & 40.57 & 36.10 & 32.81 & 36.49\\
          Rewrite & 64.25 & 60.08 & 74.98 & 66.44 & 55.63 & 59.41 & 70.50 & 61.85 & 38.83 & 34.63 & 31.09 & 34.85\\
          InfICL & 70.71 & 64.04 & 75.35 & 70.03 & 56.33 & 57.96 & 65.97 & 60.09 & 40.07 & 35.60 & 35.29 & 36.99\\
          DICL & 64.32 & 63.46 & 76.10 & 67.96 & 60.80 & 59.21 & 77.53 & 65.85 & 41.60 & 35.33 & 29.60 & 35.51\\
          \cellcolor{rowgray}DOPA*   & \cellcolor{rowgray}70.90 & \cellcolor{rowgray}63.32 & \cellcolor{rowgray}78.35 & \textbf{70.86} & \cellcolor{rowgray}62.87 & \cellcolor{rowgray}59.08 & \cellcolor{rowgray}79.00 & \textbf{66.98} & \cellcolor{rowgray}45.40 & \cellcolor{rowgray}38.10 & \cellcolor{rowgray}34.29 & \textbf{39.26} \\
    \bottomrule
  \end{tabular}
  \caption{The performance (accuracy \%) on classification tasks, * indicates that the results based on the LLM among all the datasets are statistically significant under the Wilcoxon Signed-Rank Test ($p\leq 0.05$). The calculation process of significance is presented in Appendix~\ref{app:setting}.}
  \label{tab:gpt_cls}
\end{table*}

\section{Experiment}

\subsection{Experimental Setup}
We conduct experiments on the OOD-specific benchmark BOSS~\cite{YuanCCGZCJL023}, which includes three classification tasks, Sentiment Analysis (SA), Toxicity Detection (TD), and Natural Language Inference (NLI) as well as one generation task, Named Entity Recognition (NER). All instruction templates follow the format provided in the original BOSS paper. In addition, we compare our proposed method DOPA with various baseline approaches to comprehensively demonstrate its advantages. These baselines include:
\textbf{Random}~\cite{PengDY00OT24}, \textbf{KNN}~\cite{LiuSZDCC22}, \textbf{DrICL}~\cite{Man2023}, \textbf{Rewrite}~\cite{Madine24}, \textbf{InfICL}~\cite{Vinay2024}, and \textbf{DICL}~\cite{KapuriyaKGB25}, where Rewrite refers to the data augmentation-based method, while the others are demonstration retrieval-based methods. All the methods are implemented in different LLMs to verify the adaptability of the proposed methods, including \texttt{GPT2-xl}, \texttt{Qwen3-1.7B}, \texttt{Qwen3-8B}, \texttt{Gemma2-2B}, and \texttt{LLaMA3.2-3B}, \texttt{LLaMA3.1-8B}. In addition, to investigate the performance of the proposed method on closed-source models, we also conduct experiments on \texttt{GPT4o-mini} and \texttt{GPT3.5-turbo}, and compare them with KNN, InfICL, Rewrite and DICL. Additional experimental settings can be found in Appendix~\ref{app:details}.

\begin{table*}[h]
  \setlength\tabcolsep{4.8pt}
  \renewcommand{\arraystretch}{0.3}
  \centering
  \footnotesize
    \begin{tabular}{c!{\color{white}\vrule width 1pt}
    c!{\color{white}\vrule width 1pt}c!{\color{white}\vrule width 1pt}c!{\color{white}\vrule width 1pt}>{\columncolor{colblue}}c!{\color{white}\vrule width 1pt}
    c!{\color{white}\vrule width 1pt}c!{\color{white}\vrule width 1pt}c!{\color{white}\vrule width 1pt}>{\columncolor{colblue}}c!{\color{white}\vrule width 1pt}
    c!{\color{white}\vrule width 1pt}c!{\color{white}\vrule width 1pt}c!{\color{white}\vrule width 1pt}>{\columncolor{colblue}}cc}
    \toprule
     \multirow{2}{*}{\textbf{Methods}} 
      & \multicolumn{4}{c}{\textbf{SA}} 
      & \multicolumn{4}{c}{\textbf{TD}} 
      & \multicolumn{4}{c}{\textbf{NLI}} \\ 
      \cmidrule(lr){2-5} \cmidrule(lr){6-9} \cmidrule(lr){10-13}
     & dynasent & semeval & sst & \textbf{avg} 
                  & implicit & adv & toxigen & \textbf{avg} 
                  & wanli & anli & cnli & \textbf{avg} \\   \midrule
    \rowcolor{black!8} \multicolumn{13}{c}{\textbf{\texttt{GPT4o-mini}}} \\
          KNN    & 67.67 & 60.50 & 78.17 & 68.78  & 57.63 & 57.13 & 82.00 & 65.58  & 37.67 & 39.67 & 26.83 & 34.72 \\
          Rewrite & 64.33 & 56.17 & 75.83 & 65.44 & 54.63 & 54.75 & 72.25 & 60.54 & 36.83 & 37.67 & 27.67 & 34.06 \\
          InfICL & 63.83 & 54.17 & 80.67 & 66.22  & 59.38 & 62.88 & 84.15 & 68.79  & 38.38 & 38.00 & 32.83 & 36.39 \\
          DICL & 68.67 & 62.00 & 78.00 & 69.56 & 57.38 & 54.75 & 80.63 & 64.25 & 37.00 & 36.67 & 20.50 & 31.39 \\
          \cellcolor{rowgray}DOPA*   & \cellcolor{rowgray}67.83 & \cellcolor{rowgray}62.00 & \cellcolor{rowgray}81.17 & \textbf{70.33}  & \cellcolor{rowgray}59.50 & \cellcolor{rowgray}65.25 & \cellcolor{rowgray}83.25 & \textbf{69.33}  & \cellcolor{rowgray}38.67 & \cellcolor{rowgray}40.83 & \cellcolor{rowgray}32.67 & \textbf{37.39} \\
    \midrule
    \rowcolor{black!8} \multicolumn{13}{c}{\textbf{\texttt{GPT3.5-turbo}}} \\
          KNN    & 67.17 & 60.00 & 78.83 & 68.67  & 58.38 & 60.50 & 82.23 & 67.04  & 38.00 & 40.00 & 30.83 & 36.28 \\
          Rewrite & 66.00 & 55.17 & 75.50 & 65.56 & 54.63 & 53.75 & 73.50 & 60.63 & 37.00 & 38.00 & 28.67 & 34.56 \\
          InfICL & 66.00 & 55.83 & 80.17 & 67.33  & 60.00 & 64.13 & 84.50 & \textbf{69.54}  & 37.67 & 39.50 & 32.50 & \textbf{36.56} \\
          DICL & 61.83 & 43.50 & 62.33 & 55.89 & 38.50 & 39.62 & 40.63 & 39.58 & 10.67 & 13.67 & 8.17 & 10.83 \\
          \cellcolor{rowgray}DOPA*   & \cellcolor{rowgray}68.00 & \cellcolor{rowgray}61.17 & \cellcolor{rowgray}80.50 & \textbf{69.89}  & \cellcolor{rowgray}58.50 & \cellcolor{rowgray}66.13 & \cellcolor{rowgray}82.63 & 69.08  & \cellcolor{rowgray}38.17 & \cellcolor{rowgray}39.33 & \cellcolor{rowgray}32.00 & 36.50 \\
    \bottomrule
  \end{tabular}
  \caption{The performance (accuracy \%) on classification tasks based on closed-source models.}
  \label{tab:gpt_cls2}
\end{table*}

\begin{table}[h]
  \centering
  \setlength\tabcolsep{5pt}
  \renewcommand{\arraystretch}{0.3}
  \footnotesize
    \begin{tabular}{cc!{\color{white}\vrule width 1pt}
        c!{\color{white}\vrule width 1pt}c!{\color{white}\vrule width 1pt}>{\columncolor{colblue}}c!{\color{white}\vrule width 1pt} cc!{\color{white}\vrule width 1pt}}  \toprule
      \textbf{LLMs} & \textbf{Methods} & \textbf{wnut} & \textbf{ener} & \textbf{avg} \\
      \midrule
      \multirow{5}[5]{*}{\texttt{GPT2-xl}} 
        & Random & 22.85 & 32.83 & 27.84 \\
        & KNN & 51.03 & 54.47 & 52.75 \\
        & DrICL & 19.65 & 23.29 & 21.47 \\
        & DICL & 22.95 & 31.76 & 27.36 \\
        & \cellcolor{rowgray}DOPA & \cellcolor{rowgray}51.32 & \cellcolor{rowgray}56.82 & \cellcolor{rowgray}\textbf{54.07} \\
      \midrule
      \multirow{5}[5]{*}{\texttt{Qwen3-1.7B}} 
        & Random & 25.92 & 28.71 & 27.32 \\
        & KNN & 38.74 & 50.62 & 44.68 \\
        & DrICL & 39.59 & 45.33 & 42.46 \\
        & DICL & 40.44 & 48.66 & 44.55 \\
        & \cellcolor{rowgray}DOPA & \cellcolor{rowgray}41.25 & \cellcolor{rowgray}49.62 & \cellcolor{rowgray}\textbf{45.44} \\
      \midrule
      \multirow{5}[5]{*}{\texttt{LLaMA3.2-3B}} 
        & Random & 29.54 & 41.13 & 35.34 \\
        & KNN & 33.50 & 40.97 & 37.23 \\
        & DrICL & 28.87 & 32.08 & 30.47 \\
        & DICL & 31.17 & 41.10 & 36.13 \\
        & \cellcolor{rowgray}DOPA & \cellcolor{rowgray}37.19 & \cellcolor{rowgray}41.40 & \cellcolor{rowgray}\textbf{39.29} \\
      \midrule
      \multirow{5}[5]{*}{\texttt{Gemma2-2B}} 
        & Random & 20.70 & 22.88 & 21.79 \\
        & KNN & 29.64 & 37.70 & \textbf{33.67} \\
        & DrICL & 24.30 & 27.12 & 25.71 \\
        & DICL & 28.59 & 37.05 & 32.82 \\ 
        & \cellcolor{rowgray}DOPA & \cellcolor{rowgray}29.67 & \cellcolor{rowgray}37.37 & \cellcolor{rowgray}33.52 \\
      \midrule
      \multirow{5}[5]{*}{\texttt{LLaMA3.1-8B}} 
        & Random & 42.85 & 43.15 & 43.00\\
        & KNN & 47.20 & 50.90 & 49.05 \\
        & DrICL & 47.72 & 44.78 & 46.25 \\
        & DICL & 48.14 & 50.73 & 49.44 \\ 
        & \cellcolor{rowgray}DOPA & \cellcolor{rowgray}51.66 & \cellcolor{rowgray}52.28 & \cellcolor{rowgray}\textbf{51.97} \\
      \midrule
      \multirow{5}[5]{*}{\texttt{Qwen3-8B}} 
        & Random & 58.02 & 55.78 & 56.90 \\
        & KNN & 56.48 & 56.82 & 56.65 \\
        & DrICL & 56.03 & 55.10 & 55.56 \\
        & DICL & 57.48 & 56.46 & 56.97 \\ 
        & \cellcolor{rowgray}DOPA & \cellcolor{rowgray}58.75 & \cellcolor{rowgray}56.40 & \cellcolor{rowgray}57.57 \\
      \midrule
      \multirow{5}[5]{*}{\texttt{GPT4o-mini}} 
        & Random & 42.98 & 25.20 & 34.09 \\
        & KNN & 44.13 & 24.02 & 34.08 \\
        & DrICL & 46.72 & 25.50 & 36.11 \\
        & DICL & 47.29 & 22.33 & 34.81 \\
        & \cellcolor{rowgray}DOPA & \cellcolor{rowgray}51.49 & \cellcolor{rowgray}38.19 & \cellcolor{rowgray}\textbf{44.84} \\
      \midrule
      \multirow{5}[5]{*}{\texttt{GPT3.5-turbo}} 
        & Random & 44.84 & 28.97 & 36.90 \\
        & KNN & 46.65 & 25.11 & 35.88 \\
        & DrICL & 45.30 & 28.23 & 36.77 \\
        & DICL & 29.40 & 22.22 & 25.81  \\
        & \cellcolor{rowgray}DOPA & \cellcolor{rowgray}50.00 & \cellcolor{rowgray}35.29 & \cellcolor{rowgray}\textbf{42.65} \\
      \bottomrule
    \end{tabular}
  \caption{The performance on NER tasks. }
  \label{tab:gpt_ner}
\end{table}

\begin{table}[h]
  \centering
  \setlength\tabcolsep{4.5pt}
  \renewcommand{\arraystretch}{0.1}
  \small   
  \begin{tabular}{ccc!{\color{white}\vrule width 1pt}
    c!{\color{white}\vrule width 1pt}c!{\color{white}\vrule width 1pt}>{\columncolor{colblue}}c!{\color{white}\vrule width 1pt}}
    \toprule
    \textbf{Variants} 
      & \textbf{SA} & \textbf{TD} & \textbf{NLI} & \textbf{NER} & \textbf{avg} \\
    \midrule
    \multicolumn{6}{c}{\textbf{\texttt{LLaMA3.2-3B}}} \\
    \midrule
    DOPA$_{-mah}$ & 56.59 & 58.67 & 38.66 & 37.13 & 47.76$_{\downarrow 1.64}$ \\
    DOPA$_{-sim}$ & 56.15 & 57.62 & 37.24 & 36.24 & 46.81$_{\downarrow 2.59}$ \\
    DOPA$_{-pro}$ & 55.81 & 59.36 & 38.41 & 37.21 & 47.70$_{\downarrow 1.70}$ \\
    DOPA$_{uni}$  & 57.46 & 59.67 & 38.52 & 34.57 & 47.56$_{\downarrow 1.84}$ \\
    \rowcolor{rowgray}
    DOPA          & 59.29 & 59.87 & 39.15 & 39.29 & 49.40 \\
    \midrule
    \multicolumn{6}{c}{\textbf{\texttt{Qwen3-1.7B}}} \\
    \midrule
    DOPA$_{-mah}$ & 64.33 & 58.54 & 38.92 & 44.26 & 51.51$_{\downarrow 0.87}$ \\
    DOPA$_{-sim}$ & 61.64 & 58.00 & 39.58 & 37.50 & 49.18$_{\downarrow 3.20}$ \\
    DOPA$_{-pro}$ & 61.89 & 58.31 & 39.21 & 44.60 & 51.00$_{\downarrow 1.38}$ \\
    DOPA$_{uni}$  & 63.57 & 58.14 & 38.87 & 42.17 & 50.69$_{\downarrow 1.69}$ \\
    \rowcolor{rowgray}
    DOPA          & 64.93 & 59.22 & 39.92 & 45.44 & 52.38 \\
    \bottomrule
  \end{tabular}
  \caption{Ablation study results on \texttt{LLaMA3.2-3B} and \texttt{Qwen3-1.7B}.}
  \label{tab:abl}
\end{table}

\subsection{Experimental Results}
We present the comparison results of DOPA with the aforementioned baseline methods on different LLMs in Table~\ref{tab:gpt_cls}, Table~\ref{tab:gpt_cls2}, and Table~\ref{tab:gpt_ner}. We do not compare InfICL and Rewrite on the NER task because, for token-level tasks, the influence of individual samples is difficult to quantify, and sentence rewriting may change the original entities. As an alternative, we compare with KNN, DrICL and DICL, which are not affected by the type of task. 

For classification tasks (Table~\ref{tab:gpt_cls}), DOPA exhibits noticeable performance degradation in only a few cases, demonstrating strong robustness under distribution shift. Wilcoxon signed-rank tests over nine evaluation tasks further confirm that DOPA significantly outperforms all baselines. By contrast, several recent methods fail to consistently surpass random selection in OOD settings. Notably, under \texttt{LLaMA3.2-3B}, Random sampling consistently outperforms similarity-based retrieval such as KNN, highlighting the persistent challenges posed by distribution shift and the instability of purely semantic retrieval. Methods leveraging additional signals show mixed effectiveness. DrICL, which incorporates LLM feedback to train a dense retriever, generally improves upon KNN but remains inferior to DOPA overall. The Rewrite strategy performs poorly, likely due to the absence of target-domain samples, which constrains the quality of rewritten demonstrations. InfICL achieves performance comparable to DOPA in a few settings (e.g., TD with \texttt{Qwen3-1.7B} and \texttt{GPT3.5-turbo}) but exhibits substantial instability, performing worst on NLI with \texttt{GPT2-xl} and \texttt{LLaMA3.2-3B}, and on SA with \texttt{Gemma2-2B}. Finally, the diversity-based method DICL fails to reliably alleviate distribution shift and can even degrade performance, as observed in SA with \texttt{Qwen3-1.7B} and NLI with \texttt{GPT4o-mini}.

For generative NER tasks (Table~\ref{tab:gpt_ner}), DOPA yields more pronounced performance gains, likely due to the higher difficulty of NER compared to classification tasks, which makes it more sensitive to the distribution of demonstration samples. We further observe that KNN-based retrieval benefits lightweight, locally deployable LLMs, as these models rely more heavily on external examples for guidance. In contrast, for larger models such as \texttt{GPT4o-mini} and \texttt{GPT3.5-turbo}, KNN retrieval can be detrimental. This is likely because their stronger reasoning capabilities and heightened sensitivity to distribution shifts make them more susceptible to misleading demonstrations that are semantically similar but distributionally mismatched. Overall, DOPA consistently improves performance while maintaining robustness across diverse tasks and models. Motivated by these results, we proceed to analyze the underlying mechanisms of DOPA.

\subsection{Experimental Analysis}

\begin{figure*}[h]
    \centering
    \includegraphics[width=0.99\linewidth]{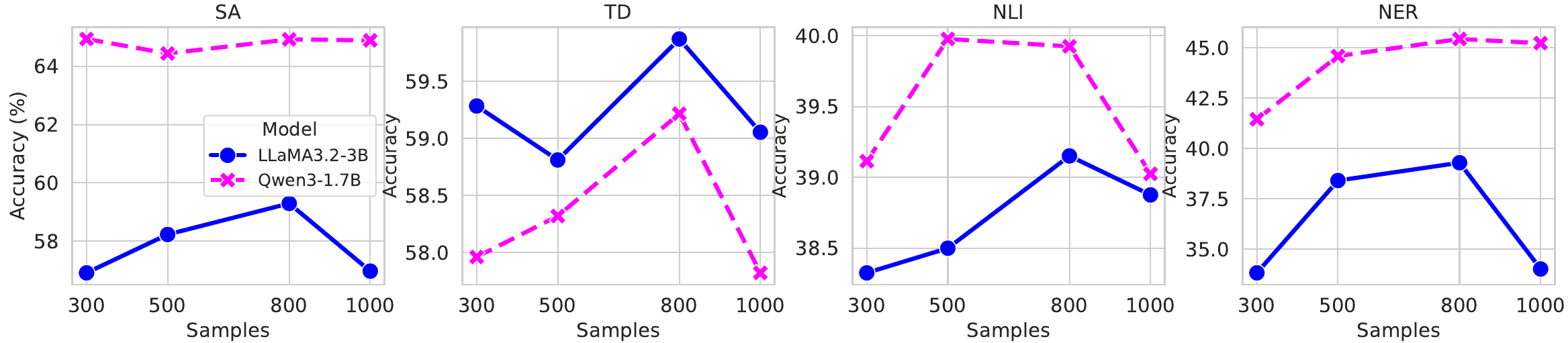}
    \caption{Performance influence of $k$ on \texttt{LLaMA3.2-3B} and \texttt{Qwen3-1.7B} across tasks.}
    \label{fig:k}
\end{figure*}

\begin{figure*}[h]
    \centering
    \includegraphics[width=0.99\linewidth]{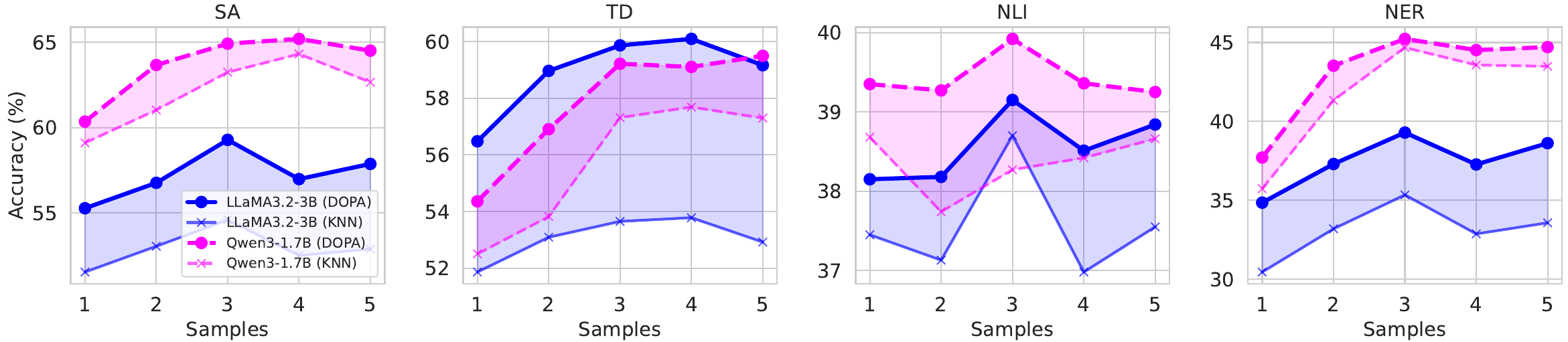}
    \caption{Performance influence of $N$ on DOPA and KNN based on \texttt{LLaMA3.2-3B} and \texttt{Qwen3-1.7B}, the shaded areas with corresponding colors indicate the performance differences. }
    \label{fig:N}
\end{figure*}

\subsubsection{Ablation Study}
To assess the contribution of each component in DOPA, we conduct an ablation study comparing several variants (Table~\ref{tab:abl}). Specifically, \textbf{DOPA${-pro}$} removes the OOD proxy and relies solely on representation similarity; \textbf{DOPA${-sim}$} discards semantic similarity and performs retrieval purely based on the OOD proxy; \textbf{DOPA${-mah}$} removes the Mahalanobis distance–based diversity constraint; and \textbf{DOPA${uni}$} replaces the LLM-based target proxy with a uniform distribution to empirically validate Lemma~\ref{lem:1}. All variants result in performance degradation, with the smallest drop observed for DOPA${-mah}$, followed by DOPA${-pro}$ and DOPA${uni}$, and the largest drop for DOPA${-sim}$. These results highlight the complementary roles of DOPA’s components: the OOD proxy enables coarse target-domain alignment, semantic similarity provides critical fine-grained filtering, and the diversity constraint promotes representative demonstrations. Notably, replacing the learned proxy with a uniform one leads to substantial degradation, supporting the effectiveness of the LLM-based proxy and the validity of Lemma~\ref{lem:1}. Overall, proxy-based filtering and diversity-aware retrieval jointly contribute to improved demonstration quality and model performance.

\subsubsection{Exploration of $k$}
We investigate the impact of varying $k \in {300, 500, 800, 1000}$ on demonstration selection and model performance (Figure~\ref{fig:k}). Results indicate that small $k$ values limit example diversity and hinder generalization, whereas larger $k$ values introduce noisy or redundant demonstrations that may degrade performance. Based on systematic evaluations across multiple tasks and datasets, we adopt $k=800$ as a unified setting, as it provides a favorable balance between diversity and noise, despite not being optimal in all cases. Using a fixed $k$ also simplifies the selection process and ensures more consistent and comparable performance across tasks.

\subsubsection{Exploration of $N$}
We study the effect of varying $N \in {1,2,3,4,5}$ on model performance (Figure~\ref{fig:N}), using KNN as a task-agnostic and stable baseline. Here, $N$ corresponds to $N \times |Y|$ demonstration samples. Performance generally improves with more demonstrations before gradually saturating~\cite{MinLHALHZ22}, with the saturation point varying across models and tasks. For instance, in SA, \texttt{Qwen3-1.7B} peaks at $N=4$, while \texttt{LLaMA3.2-3B} peaks at $N=3$. Across all $N$ settings, DOPA consistently outperforms KNN by a substantial margin. For simplicity and consistency in the main experiments, we fix $N=3$.

\subsubsection{Visualization}
\begin{figure}[h]
  \centering
  \begin{subfigure}[t]{0.245\textwidth}
    \centering
    \includegraphics[width=\textwidth]{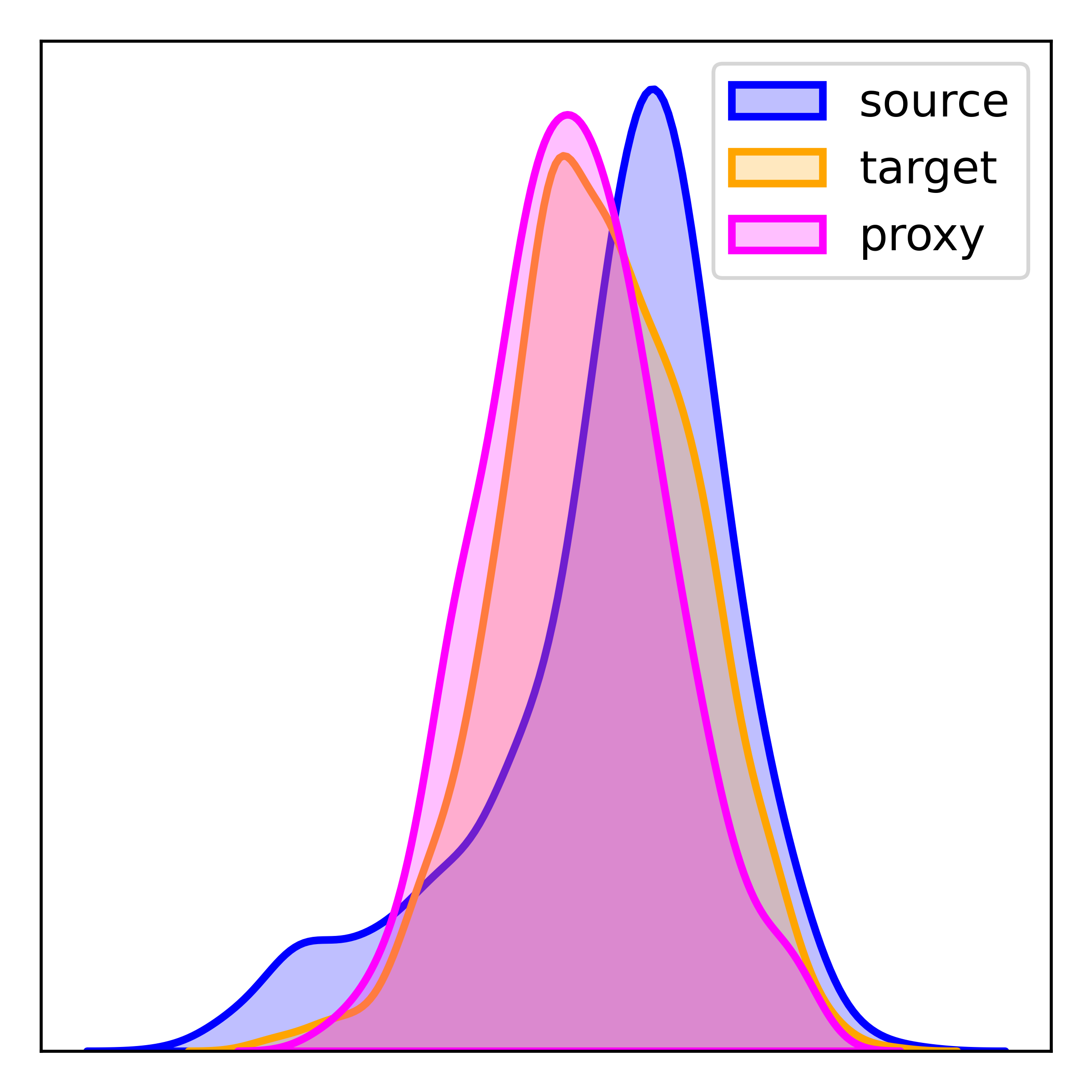}
    \caption{KDE visualization.}
    \label{fig:sst_dke}
  \end{subfigure}
  \begin{subfigure}[t]{0.23\textwidth}
    \centering
    \includegraphics[width=\textwidth]{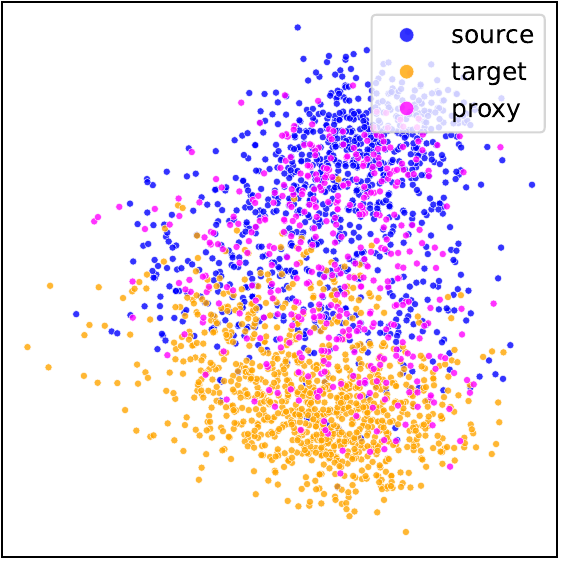}
    \caption{t-SNE visualization.}
    \label{fig:sst_tsne}
  \end{subfigure}
  \begin{subfigure}[t]{0.48\textwidth}
    \centering
    \includegraphics[width=\textwidth]{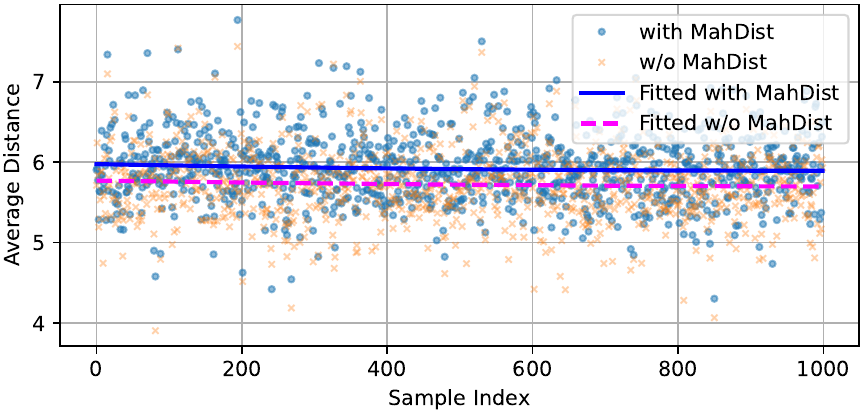}
    \caption{Euclidean distance comparison to target domain samples for retrieval results with and without the diversity constraint (with MahDist and w/o MahDist).}
    \label{fig:sst_vis2}
  \end{subfigure}
  \caption{Different visualization results on \textit{sst}. }
  \label{fig:sst_vis}
\end{figure}
We validate the effectiveness of the proposed OOD proxy through visualization of proxy-based sample selection. Specifically, we characterize the behaviors of $\mathcal{D}_S$, $\hat{\mathcal{D}}_S$, and $\mathcal{D}_T$ by computing BERT-based energy scores~\cite{LiuWOL20} and estimating their distributions via kernel density estimation (KDE)~\cite{wkeglarczyk2018kernel}. We adopt energy-score distributions rather than the commonly used t-SNE visualizations, as representational proximity alone does not adequately reflect OOD tendencies. As shown in Figure~\ref{fig:sst_vis}, the proxy-selected samples exhibit an energy-score distribution that more closely matches and overlaps with that of the target domain (Figure~\ref{fig:sst_dke}), indicating effective source-to-target alignment. In contrast, t-SNE visualizations (Figure~\ref{fig:sst_tsne}) show that proxy-selected samples remain closer to the source domain in representation space, suggesting residual semantic distance from the target domain. This discrepancy further explains the results in Table~\ref{tab:gpt_cls}, demonstrating that semantic similarity–based retrieval (e.g., KNN) is insufficient for OOD adaptation, whereas DOPA effectively addresses the limitations of purely representation-based methods.

Additionally, to assess the effectiveness of the diversity constraint, we compute the Euclidean distances between retrieved demonstrations and their corresponding test samples for the first 1,000 test instances, under both with- and without-MahDist settings. Figure~\ref{fig:sst_vis2} presents the corresponding fitted distance distributions. The curve without MahDist consistently lies below that with MahDist, indicating that the diversity constraint encourages more varied retrieval. Importantly, the two curves remain close, suggesting that this increased diversity is controlled and does not introduce excessive semantic drift. In summary, the visualization results provide evidence for the effectiveness of DOPA from two perspectives: it helps retrieve demonstrations that exhibit similar behavior to target domain samples while maintaining high diversity, thereby enhancing the performance of ICL. We also observe similar trends across the remaining datasets. Additional visualization results are in Appendix~\ref{app:vis}.

\section{Conclusion}
This paper shows that OOD proxies can effectively retrieve source samples aligned with the target domain under distribution shift. Based on this insight, we propose DOPA, a target-free framework enhanced with a diversity constraint to mitigate proxy bias. Experiments across multiple LLMs confirm its effectiveness. Future work will focus on more robust proxy estimation for unknown target domains.

\section*{Limitations}
Since target-domain data is unavailable, the constructed target-domain proxies may not fully capture the true distributional characteristics of the target dataset, inevitably introducing approximation errors. Consequently, developing more accurate and robust target-domain proxy construction methods remains an important direction for future work. In addition, evaluating DOPA across a broader range of LLM families would further clarify its generality and adaptability.

\section*{Acknowledgements}
This work is supported by the National Natural Science Foundation of China (NSFC): “Research on Understanding Ancient Characters Based on Multi-modal Large Models” (Grant No. 62476111), China Postdoctoral Science Foundation Funded Project (Grant No. 2024M761122), the Scientific Research Project of the Education Department of Jilin Province (Grant No. JJKH20261299KJ) and the Industry University Research Innovation Fund of the Ministry of Education project "Research and Application of an Integrated Teaching Model for Human centered Artificial Intelligence" (Grant No. 2022XF017).

\bibliography{custom}

\appendix

\section{Theoretical Analysis and Proof} \label{app:th}
The following provides a detailed proof of the boundedness of proxy errors.
\begin{proof}
We use shorthand notation: let $P_t := P_{\mathrm{target}}$, $P_s := P_{\mathrm{source}}$, $P_t^p := P_{\mathrm{target}}^{\mathrm{proxy}}$, and $P_s^p := P_{\mathrm{source}}^{\mathrm{proxy}}$.

We aim to bound the log-likelihood ratio error:
\[
\Delta(x) := \left| \log \frac{P_t(x)}{P_s(x)} - \log \frac{P_t^p(x)}{P_s^p(x)} \right|
\]

Applying the triangle inequality:
\[
\begin{aligned}[t]
& \Delta(x) = \left| \log \frac{P_t(x)}{P_t^p(x)} - \log \frac{P_s(x)}{P_s^p(x)}  \right| \\
&\le \left| \log \frac{P_t(x)}{P_t^p(x)} \right|
   +  \left| \log \frac{P_s(x)}{P_s^p(x)} \right|
\end{aligned}
\]

We now upper bound each term. Then, from the definition of KL divergence:
\[
D_{\mathrm{KL}}(P_t \| P_t^p) = \sum_x P_t(x) \log \frac{P_t(x)}{P_t^p(x)} \leq \varepsilon_t
\]

Now, suppose for some $x$ we have $P_t(x) \ge m_t$ and
\[
\left| \log \frac{P_t(x)}{P_t^p(x)} \right| > \frac{\varepsilon_t}{m_t}
\]
Then,
\[
P_t(x) \cdot \left| \log \frac{P_t(x)}{P_t^p(x)} \right| > m_t \cdot \frac{\varepsilon_t}{m_t} = \varepsilon_t
\]
This contradicts the assumption \( D_{\mathrm{KL}}(P_t \| P_t^p) \leq \varepsilon_t \). Therefore, for all $x$:
\[
\left| \log \frac{P_t(x)}{P_t^p(x)} \right| \leq \frac{\varepsilon_t}{m_t}
\]

Analogously, we obtain:
\[
\left| \log \frac{P_s(x)}{P_s^p(x)} \right| \leq \frac{\varepsilon_s}{m_s}
\]

Combining the two bounds:
\[
\Delta(x) \leq \frac{\varepsilon_t}{m_t} + \frac{\varepsilon_s}{m_s},
\]
the proof of the theorem is complete.
\end{proof}

The theorem shows that if the KL-divergence between the true distribution and its proxy is sufficiently small, and the probability mass at each point is lower bounded, then the deviation in log-probability ratios is controllable in expectation. Therefore, a properly constructed proxy distribution yields bounded error in tasks such as density ratio estimation or scoring, which verifies the effectiveness and reliability of using proxies.

Moreover, some methods propose a general approach by replacing the target-domain proxy with a uniform distribution. However, this strong assumption may lead to suboptimal solutions. Accordingly, we introduce Lemma~\ref{lem:1} to illustrate the limitations of using a uniform distribution.
\begin{proof}[Proof of Lemma~\ref{lem:1}]
We consider the case where the proxy distribution for the target domain is chosen as the uniform distribution over the support $\mathcal{X}$:
\[
P_t^p(x) = \frac{1}{|\mathcal{X}|} \quad \text{for all } x \in \mathcal{X}
\]

From Theorem~\ref{th:1}, the error in the log-likelihood ratio satisfies:
\begin{align*}
\left| \log \frac{P_t(x)}{P_s(x)} - \log \frac{P_t^p(x)}{P_s^p(x)} \right| \\
\leq \left| \log \frac{P_t(x)}{P_t^p(x)} \right| + \left| \log \frac{P_s(x)}{P_s^p(x)} \right|
\end{align*}

We now focus on bounding the first term with the uniform proxy:
\begin{align*}
    A_t(x) := \left| \log \frac{P_t(x)}{P_t^p(x)} \right| = \left| \log \left( P_t(x) \cdot |\mathcal{X}| \right) \right|  \\ =  \left| \log P_t(x) + \log |\mathcal{X}| \right|
\end{align*}

From the definition of KL divergence between $P_t$ and uniform distribution $U$:
\begin{align*}
D_{\mathrm{KL}}(P_t \| U) 
&= \sum_x P_t(x) \log \frac{P_t(x)}{1/|\mathcal{X}|}  \\
&= \sum_x P_t(x) [\log P_t(x) + \log |\mathcal{X}|] \\
&= \log |\mathcal{X}| - H(P_t)
\end{align*}
where \( H(P_t) := -\sum_x P_t(x)\log P_t(x) \) is the Shannon entropy of $P_t$.

Now suppose that $P_t(x) \ge m_t > 0$ for all $x$. Following the same logic as in the proof of Theorem~\ref{th:1}, we know that if:
\[
\left| \log \frac{P_t(x)}{P_t^p(x)} \right| > \frac{D_{\mathrm{KL}}(P_t \| U)}{m_t}
\]
Then this point would contribute more than \( D_{\mathrm{KL}}(P_t \| U) \) to the KL divergence, leading to a contradiction. Therefore, for all \( x \):
\[
\left| \log \frac{P_t(x)}{P_t^p(x)} \right| \leq \frac{D_{\mathrm{KL}}(P_t \| U)}{m_t} = \frac{\log |\mathcal{X}| - H(P_t)}{m_t}
\]

Substituting into the total bound in Theorem~\ref{th:1}, we obtain:
\[
\left| \log \frac{P_t(x)}{P_s(x)} - \log \frac{P_t^p(x)}{P_s^p(x)} \right|
\leq \frac{\log |\mathcal{X}| - H(P_t)}{m_t} + \frac{\varepsilon_s}{m_s}
\]

This upper bound is typically looser than the one obtained when $P_t^p$ approximates $P_t$ well (i.e., KL divergence is small), since $\log |\mathcal{X}| - H(P_t)$ can be large when $P_t$ is sharply peaked.

\end{proof}

\begin{figure*}[h]
  \centering
  \begin{subfigure}[t]{0.23\textwidth}
    \centering
    \includegraphics[width=\textwidth]{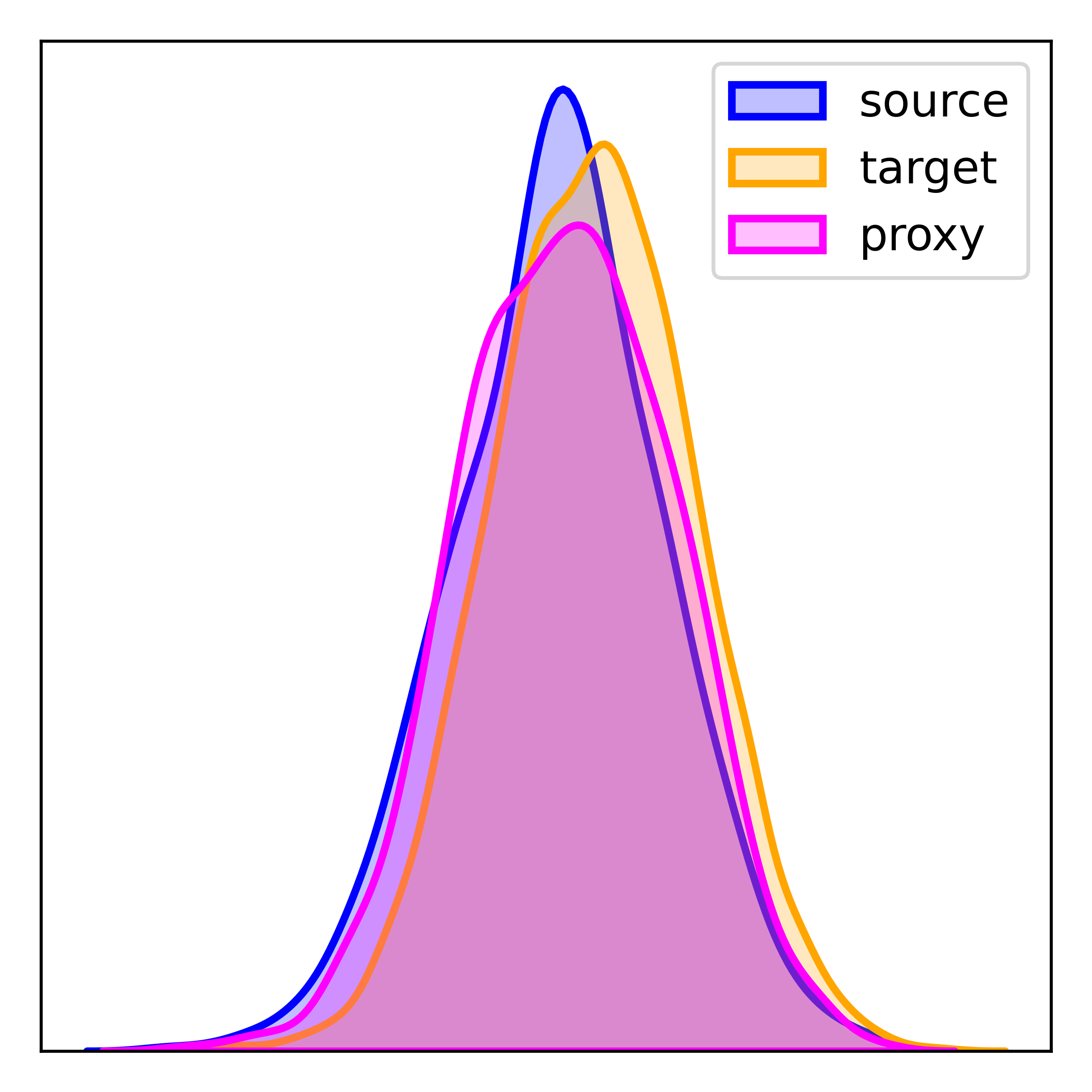}
    \caption{dynasent}
  \end{subfigure}
  \begin{subfigure}[t]{0.23\textwidth}
    \centering
    \includegraphics[width=\textwidth]{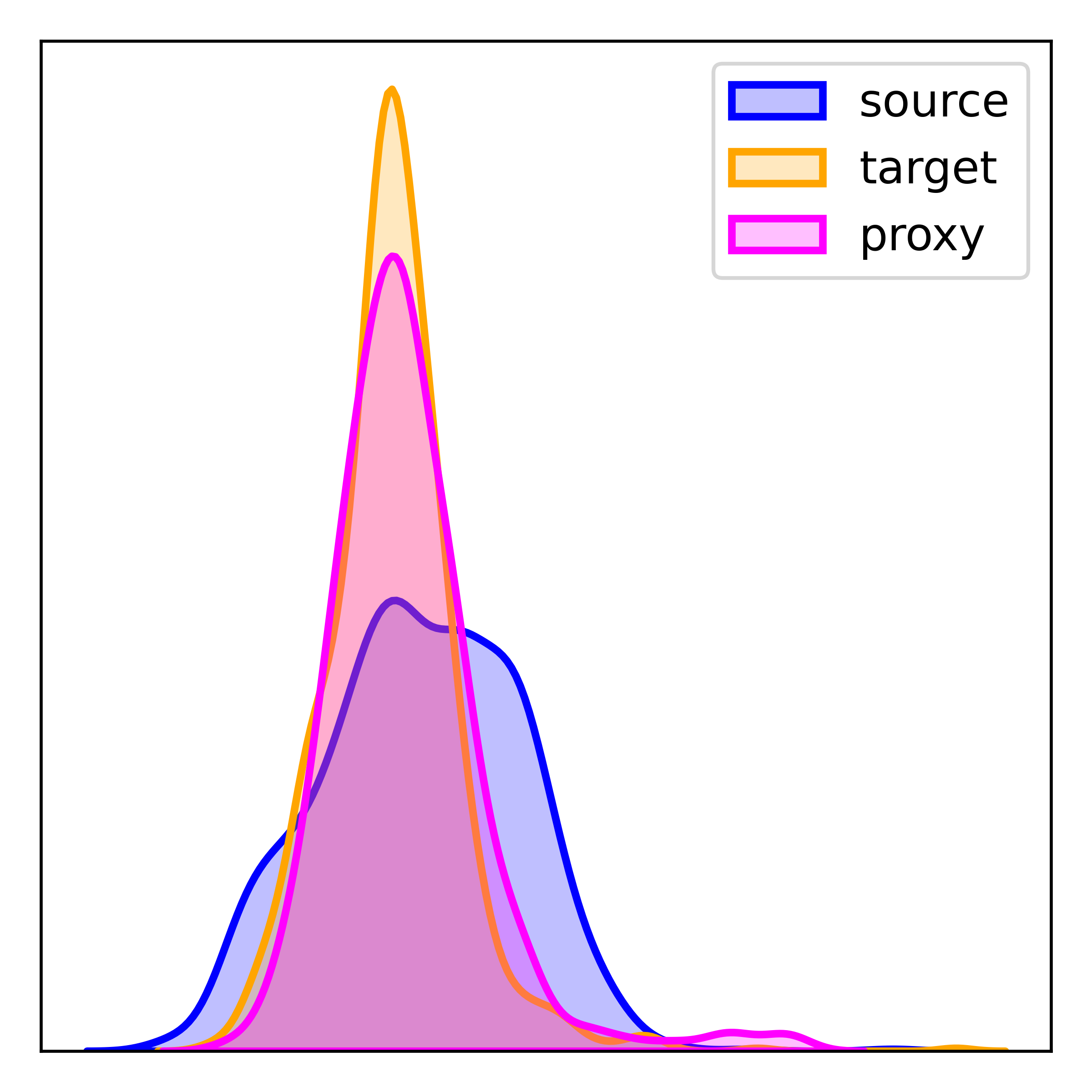}
    \caption{semeval}
  \end{subfigure}
    \begin{subfigure}[t]{0.23\textwidth}
    \centering
    \includegraphics[width=\textwidth]{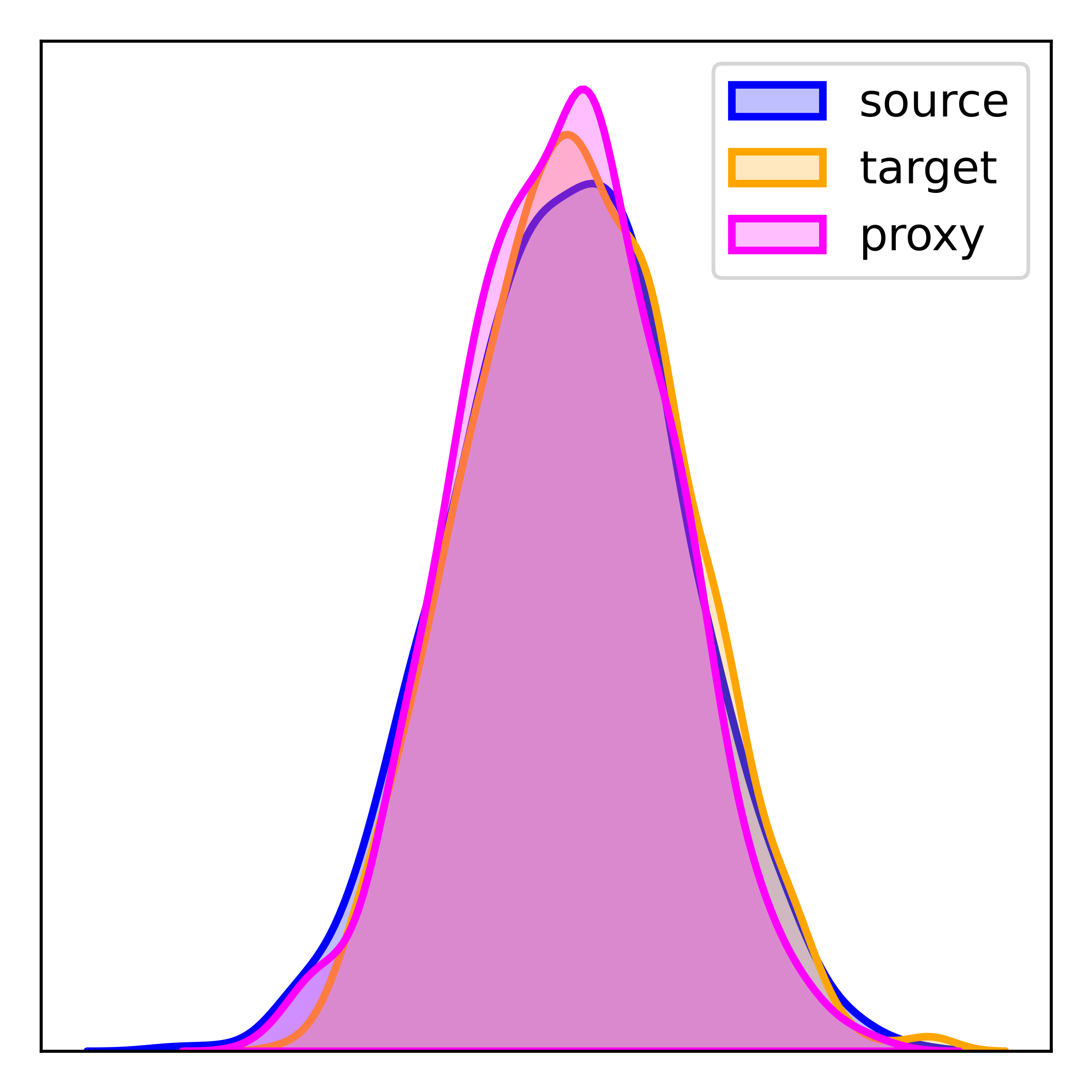}
    \caption{adv\_civil}
  \end{subfigure}
  \begin{subfigure}[t]{0.23\textwidth}
    \centering
    \includegraphics[width=\textwidth]{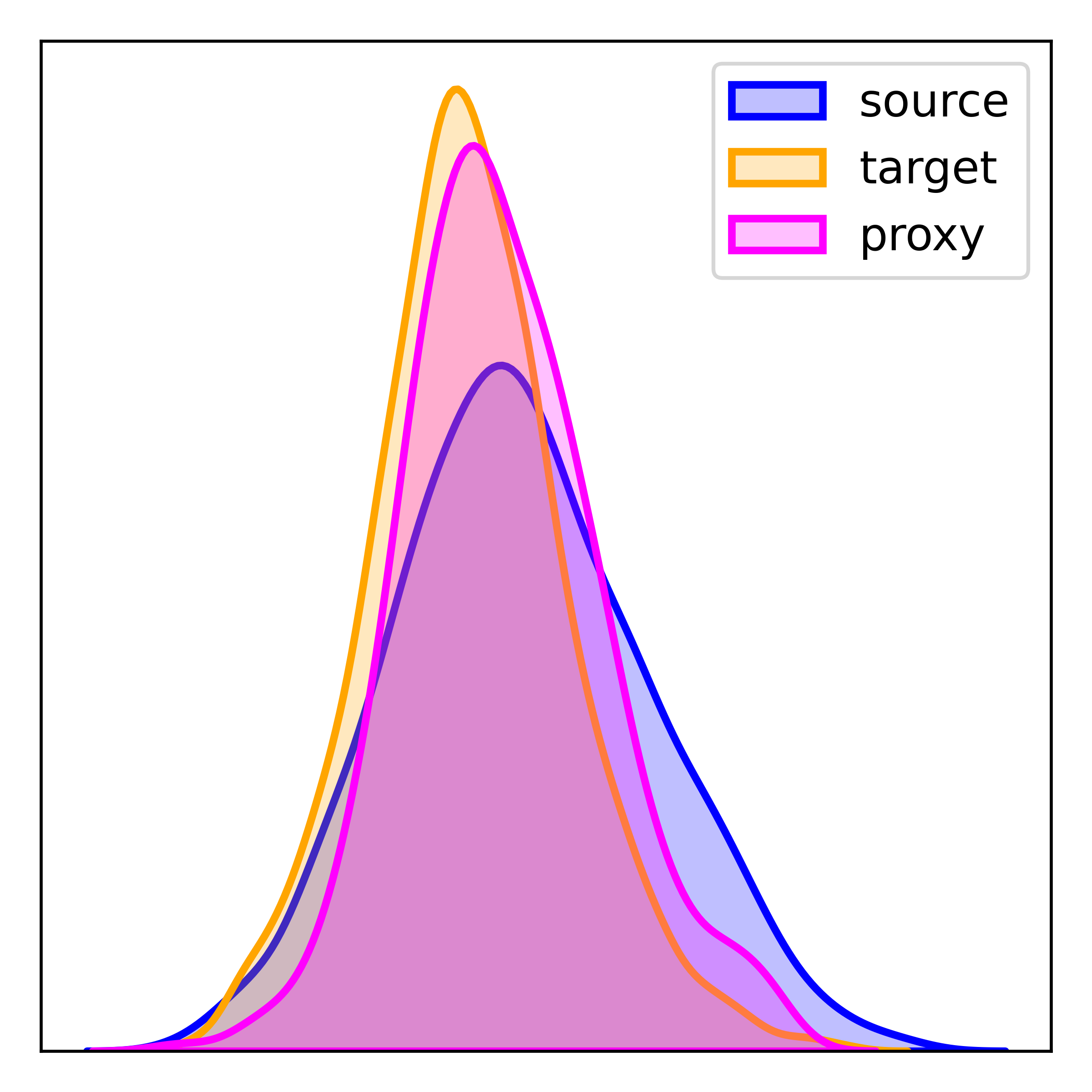}
    \caption{implicit\_hate}
  \end{subfigure}
  
    \begin{subfigure}[t]{0.23\textwidth}
    \centering
    \includegraphics[width=\textwidth]{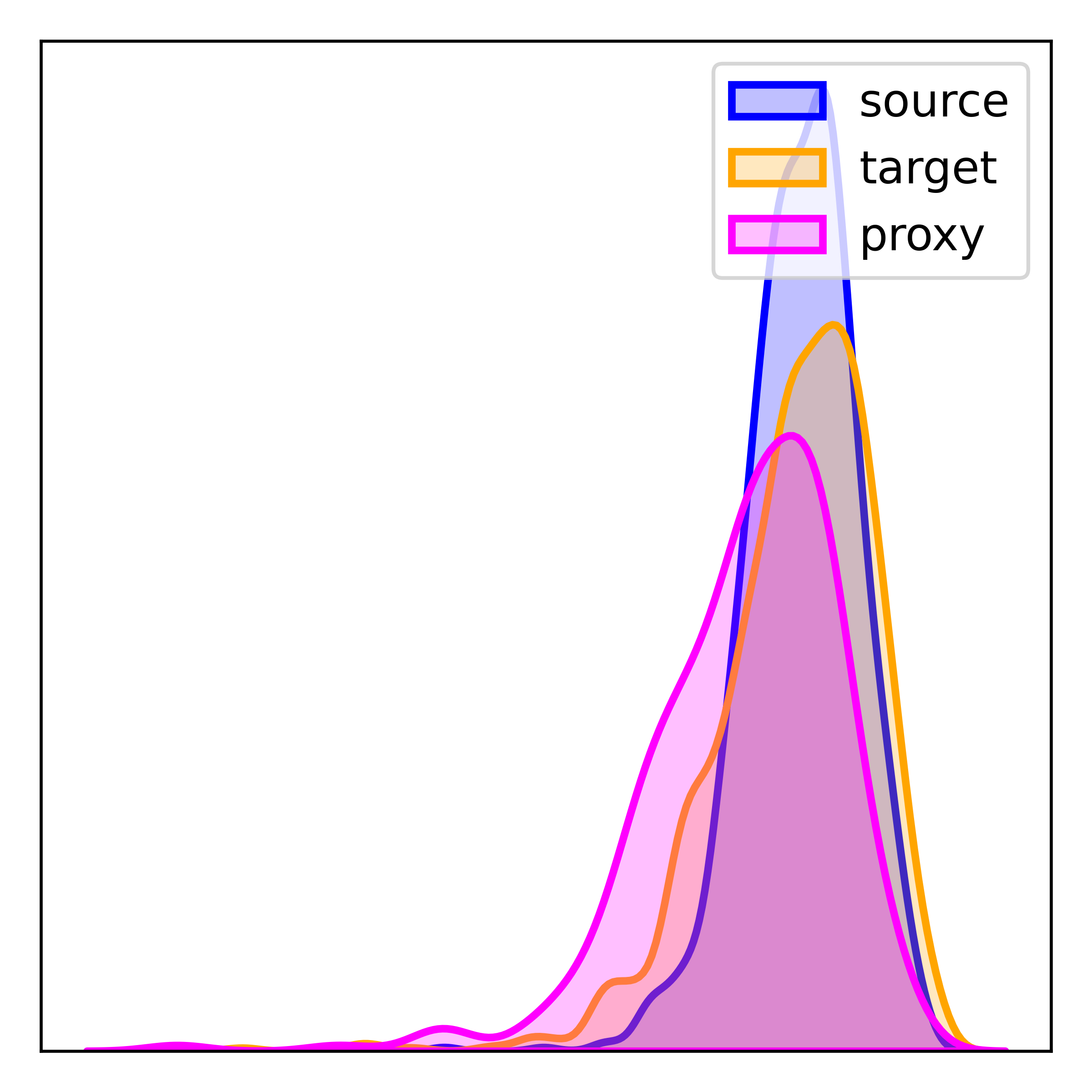}
    \caption{toxigen}
  \end{subfigure}
  \begin{subfigure}[t]{0.23\textwidth}
    \centering
    \includegraphics[width=\textwidth]{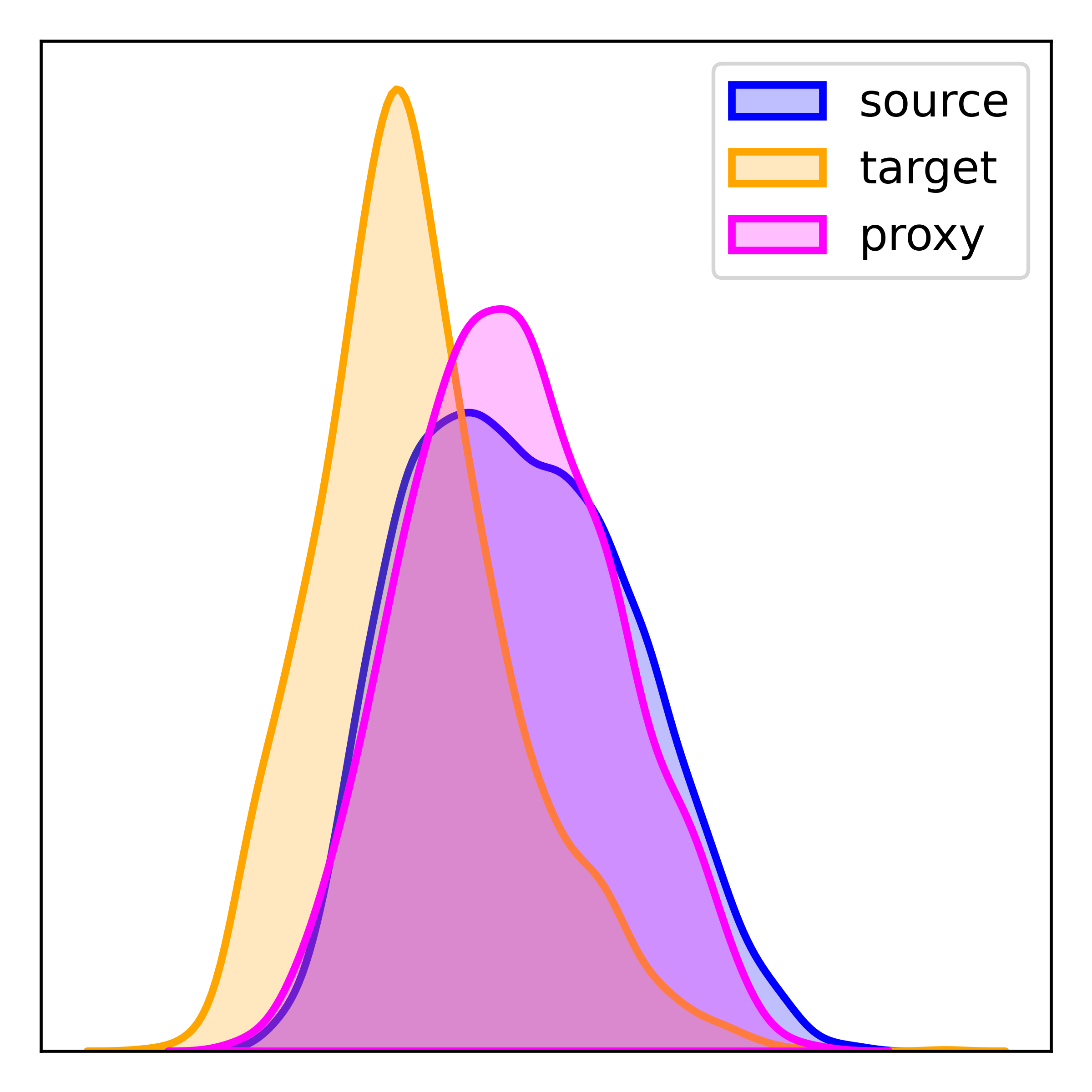}
    \caption{anli}
  \end{subfigure}
    \begin{subfigure}[t]{0.23\textwidth}
    \centering
    \includegraphics[width=\textwidth]{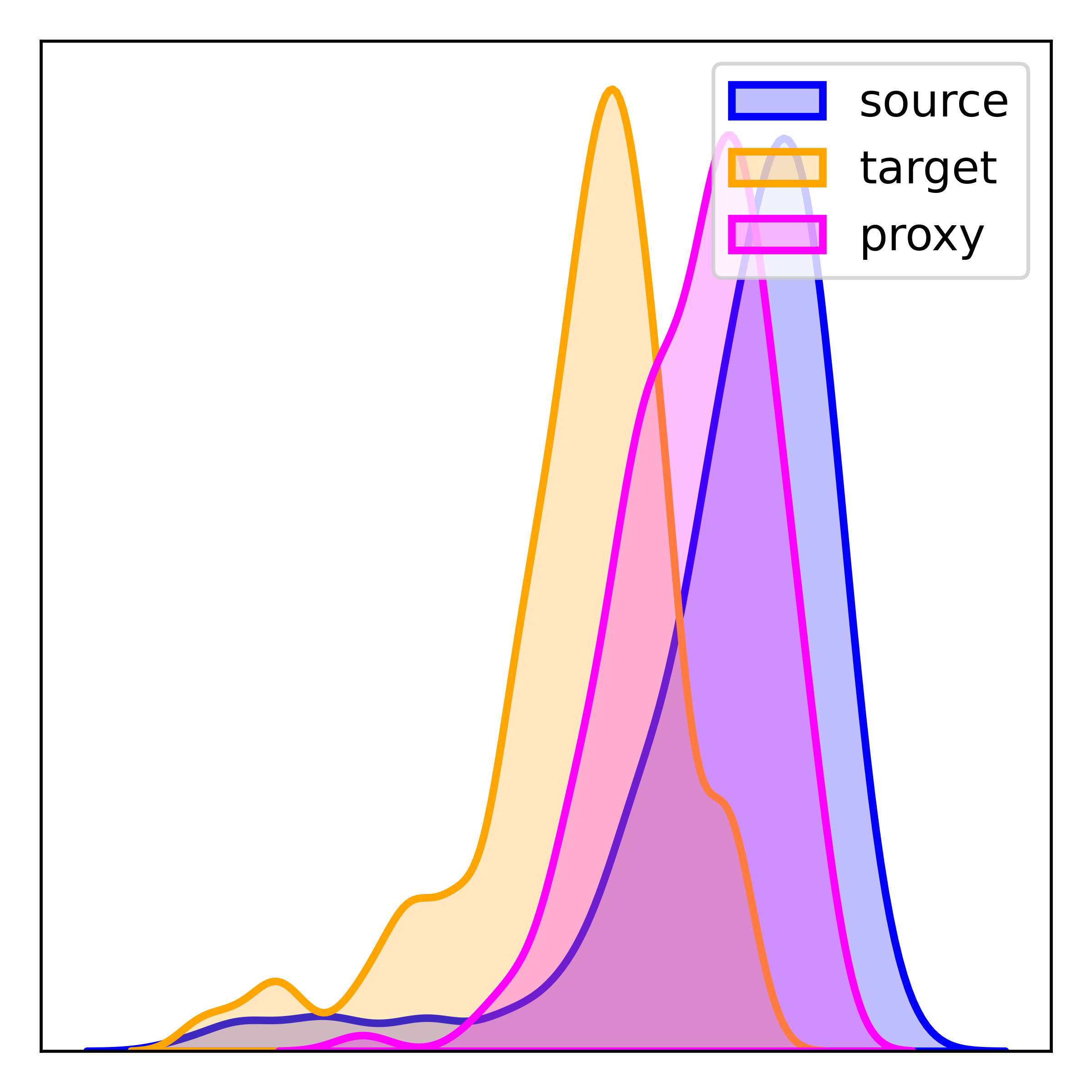}
    \caption{contract\_nli}
  \end{subfigure}
  \begin{subfigure}[t]{0.23\textwidth}
    \centering
    \includegraphics[width=\textwidth]{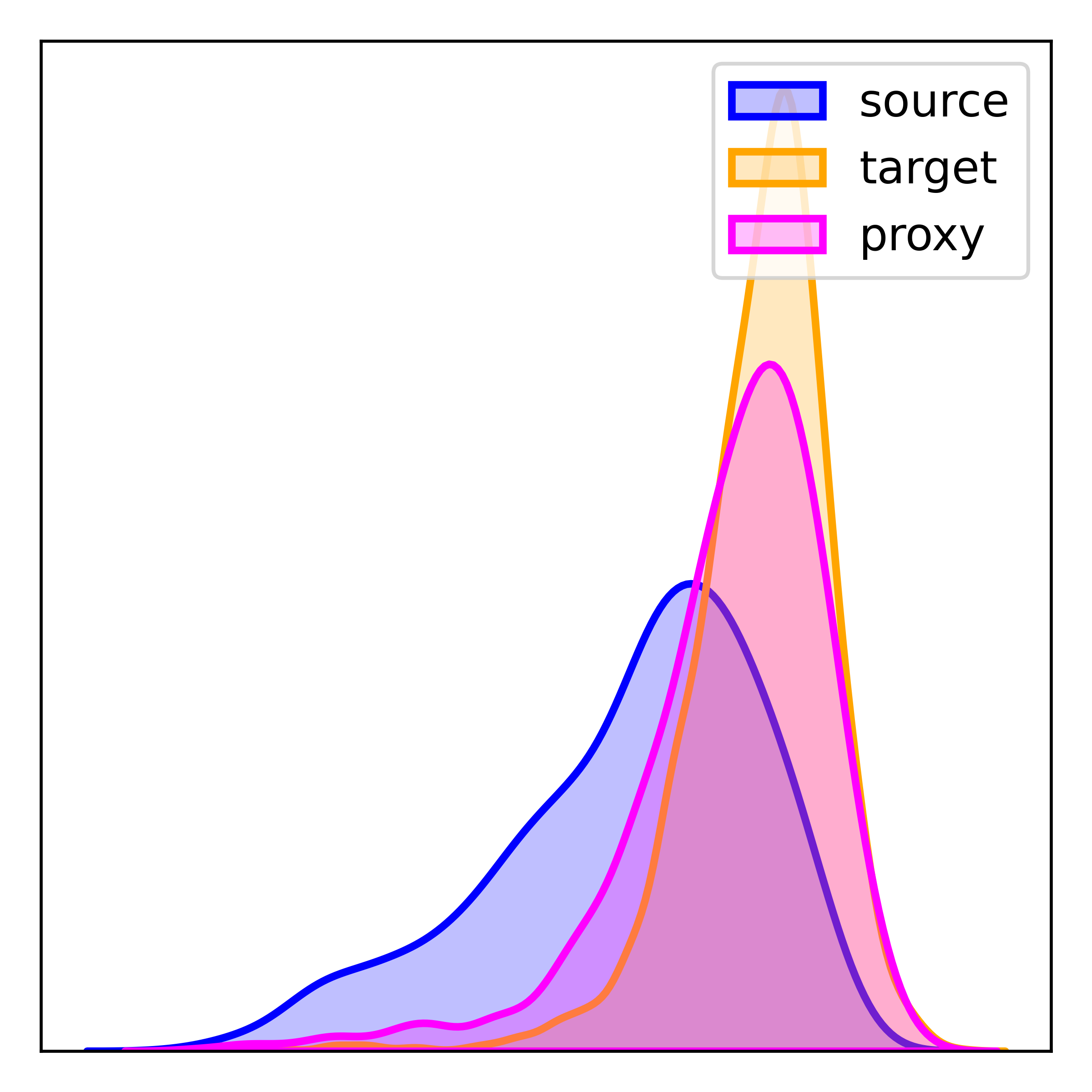}
    \caption{wanli}
  \end{subfigure}
  \caption{Different KDE visualization results on all classification tasks. }
  \label{fig:vis_all}
\end{figure*}

\section{Experimental Details}  \label{app:details}

\subsection{Dataset Details}
We focus on four core NLP tasks from BOSS~\cite{YuanCCGZCJL023}, a benchmark suite specifically designed to evaluate the robustness of language models under OOD scenarios: Sentiment Analysis (SA), Toxic Detection (TD), Natural Language Inference (NLI), and Named Entity Recognition (NER). To balance the number of samples, we randomly select 3,000 training samples per class from the original in-distribution dataset for SA and NLI, and 5,000 training samples per class for TD. Accordingly, for testing, we randomly sample up to 1,000 instances per class from the target domain for SA and NLI, and 1,500 test samples per class for TD. For the NER task, we select 10,000 samples from source dataset that contain only ``Location", ``Organization", or ``Person" entities to unify the label space and select all eligible samples from the target domain for testing. In our experiments, we do not use the conll dataset because it contains a large number of annotation errors, which could lead to unreliable and unmeasurable outcomes for model evaluation.

\subsection{Baseline Details}
We provide a detailed introduction of the baseline methods used in this section. 
\begin{itemize}
    \item Random~\cite{PengDY00OT24}. We randomly select the required number of samples from the source domain to construct demonstrations. To reduce performance variance caused by randomness, we repeat this process five times and report the average results for comparison.

    \item KNN~\cite{LiuSZDCC22}. We use the SimCSE representations of samples as the retrieval basis and construct demonstrations by selecting the top nearest samples to the test sample in the representation space.

    \item DrICL~\cite{Man2023}. We first use KNN to select the top 30 candidate samples that are most similar to the test sample. These candidates are then ranked by quantifying their individual contributions to the LLM’s actual predictions (We use LLaMA3.2-3B in LLMs that can not be deployed locally). The top 10 are treated as positive examples and the bottom 10 as negative ones to train a dual-encoder neural retriever, GTR~\cite{ni2022large}, which is subsequently used for demonstration retrieval.

    \item Rewrite~\cite{Madine24}. We perform KNN-based demonstration retrieval and rewrite the retrieved samples according to the style of the test sample, so that the demonstrations better align with the target domain. In contrast to the original method, we adapt the rewriting strategy under a strict target-unavailable setting, where only a single test instance is exposed at a time, rather than a set of target samples.

    \item InfICL~\cite{Vinay2024}. It estimates the influence of each candidate demonstration on the model’s prediction for a given test input, and to select those demonstrations that have the most beneficial effect. By leveraging gradient-based influence approximations, the method identifies which demonstrations most positively affect the model’s output distribution without requiring extensive evaluation over all combinations. 

    \item DICL~\cite{KapuriyaKGB25}. It employs Maximum Marginal Relevance (MMR), which jointly considers the relevance between the input and samples as well as the mutual dissimilarity among them. This allows the selected context to maintain high relevance with greater diversity.
    
\end{itemize}

\subsection{More Experimental Settings} \label{app:setting}
To prevent potential bias caused by an imbalanced number of samples per label in the demonstrations, we retrieve the same number of samples $N$ for each label. Therefore, for classification tasks, the total number of demonstrations is $N \times |Y|$, where $|Y|$ is the number of labels. However, for generative tasks that do not involve specific class labels, we directly set the number of demonstrations to $N$. For classification tasks, we set the number of demonstrations $C$ in the initial demonstration set to $|Y|$, while for generative tasks, we directly set $C$ to 1. For instruction fine-tuning, we use the source domain data and convert it into training samples following the instruction format of BOSS. During training, we apply LoRA with a learning rate of 1e-5 for one epoch. For \texttt{GPT4o-mini} and \texttt{GPT3.5-turbo}, we use the APIs provided by xi-ai\footnote{https://api.xi-ai.cn/}. For experiments involving closed-source LLMs (e.g., GPT-based APIs), we use Llama 3.2-3B as the proxy model for both the source and target domains. While this proxy cannot perfectly replicate the behavior of the closed-source model, it provides a practical and consistent reference when internal representations and training access are unavailable. 

To compare the performance of DOPA and baselines across multiple datasets, we employ the Wilcoxon Signed-Rank Test which is widely used for model comparison across multiple benchmarks~\cite{Demsar06}. This non-parametric statistical test is specifically designed for paired samples and does not assume normality of the underlying distribution. In our setting, the paired observations correspond to the performance scores of the two models (DOPA and any other baseline) on the same datasets. If DOPA shows statistically significant improvements ($p \leq 0.05$) over all baselines, we denote it as DOPA*.

Model Parameters Download: \texttt{GPT2-xl}, \texttt{Qwen3-1.7B}\footnote{https://huggingface.co/Qwen/Qwen3-1.7B}, \texttt{Qwen3-8B}\footnote{https://huggingface.co/Qwen/Qwen3-8B}, \texttt{Gemma2-2B}\footnote{https://huggingface.co/google/gemma-2b}, and \texttt{LLaMA3.2-3B}\footnote{https://huggingface.co/meta-llama/Llama-3.2-3B}, \texttt{LLaMA3.1-8B}\footnote{https://huggingface.co/meta-llama/Llama-3.1-8B}.

\section{Proxy Sensitivity}

\begin{table*}[h]
\centering
\small
\begin{tabular}{l l c c c c}
\toprule
\textbf{Source Proxy} & \textbf{Target Proxy} & \textbf{dynasent} & \textbf{semeval} & \textbf{sst} & \textbf{average} \\
\midrule
\texttt{Llama3.2-3B} & \texttt{Llama3.2-3B} & 55.71 & 53.28 & 68.88 & 59.29 \\
\texttt{Llama3.2-3B} & \texttt{Llama3.1-8B} & 55.38 & 50.28 & 66.45 & 57.37 \\
\texttt{Llama3.2-3B} & \texttt{Qwen3-1.7B}  & 56.54 & 45.86 & 67.76 & 56.72 \\
\midrule
\multicolumn{2}{c}{KNN}  & 52.63 & 45.76 & 65.42 & 54.60 \\
\bottomrule
\end{tabular}
\caption{Performance comparison under different target proxy settings.}
\label{tab:proxy_results}
\end{table*}

We evaluate the proxy sensitivity of DOPA by replacing the target proxy with LLMs that differ from those in the source proxy. Specifically, we show the replacement results using \texttt{Llama3.1-8B} and \texttt{Qwen3-1.7B} for the SA task in Table~\ref{tab:proxy_results}. The results show that while proxy mismatches may affect performance, DOPA still consistently outperforms KNN. This provides further evidence supporting the robustness of the proxy-based design. Additionally, we found that using models from different series (\texttt{Qwen3-1.7B}) yields inferior replacement results compared to using models from the same series (\texttt{Llama3.1-8B}), indicating that the source domain and target domain proxies should be as similar as possible to ensure accurate quantification of the characteristics of different samples.

\section{Experiments on More Tasks}

\begin{table*}[htbp]
\centering
\small
\begin{tabular}{llcc|cc|cc|c}
\toprule
\multirow{2.5}{*}{\textbf{Model}} & \multirow{2.5}{*}{\textbf{Method}} 
& \multicolumn{2}{c|}{\textbf{advQA}} 
& \multicolumn{2}{c|}{\textbf{searchQA}} 
& \multicolumn{2}{c|}{\textbf{newsQA}} 
& \multirow{2.5}{*}{\textbf{Avg}} \\
\cmidrule(lr){3-4}\cmidrule(lr){5-6}\cmidrule(lr){7-8}
& & EM & F1 & EM & F1 & EM & F1 & \\
\midrule
\multirow{2}{*}{\texttt{LLaMA3.2-3B}}
& KNN  & 30.5 & 54.94 & 33.1  & 52.46 & 32.0 & 58.61 & 43.60 \\
& DOPA & 31.5 & 55.49 & 34.5  & 53.72 & 33.2 & 59.97 & 44.73 \\
\midrule
\multirow{2}{*}{\texttt{Qwen3-1.7B}}
& KNN  & 31.7 & 51.21 & 39.35 & 51.96 & 37.3 & 61.45 & 45.50 \\
& DOPA & 32.7 & 52.65 & 40.95 & 52.56 & 37.8 & 61.49 & 46.35 \\
\bottomrule
\end{tabular}
\caption{Experimental results on QA tasks.}
\label{tab:QA}
\end{table*}

We further briefly evaluate the effectiveness of DOPA on extractive Question Answering (QA) tasks in the BOSS Benchmark. During this process, 10,000 samples are drawn from each dataset, regardless of whether it belongs to the source domain or the target domain. We adopt Exact Match (EM) and F1 as evaluation metrics. EM measures whether the predicted answer exactly matches the ground-truth span, reflecting the model's strict answer prediction accuracy. F1 computes the token-level overlap between the predicted and gold answers, providing a softer evaluation of partial matching quality and semantic coverage.

As shown in Table \ref{tab:QA}, DOPA consistently outperforms the KNN baseline across all datasets and model settings, demonstrating its effectiveness on extractive QA tasks. Specifically, for LLaMA3.2-3B, DOPA improves the average performance from 43.60 to 44.73, yielding a gain of 1.13 points. For Qwen3-1.7B, DOPA boosts the average score from 45.50 to 46.35, achieving an improvement of 0.85 points. A closer examination of individual datasets further shows that DOPA delivers consistent improvements on advQA, searchQA, and newsQA, indicating strong robustness across diverse domains and question distributions. These results validate that the effectiveness of DOPA extends beyond classification and NER tasks and generalizes well to more challenging extractive QA scenarios.

\section{Case Study}
\begin{figure}[h]
    \centering
    \includegraphics[width=1\linewidth]{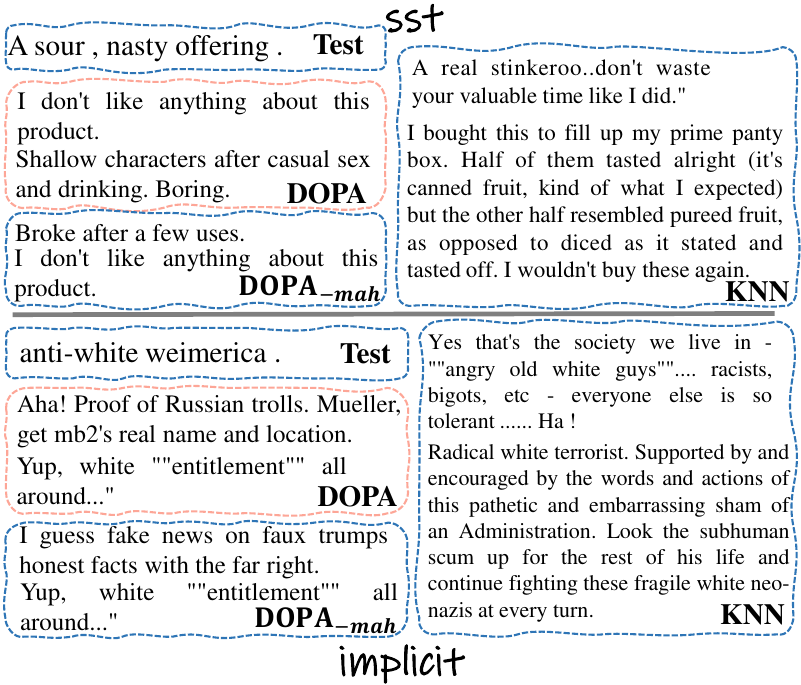}
    \caption{Case study on sst and implicit.}
    \label{fig:case}
\end{figure}
The examples in Figure~\ref{fig:case} illustrate that both DOPA and DOPA$_{-mah}$ select samples with stylistic expressions closely aligned with the test inputs, capturing similar tone, sentence structure, and emotional/toxicity intensity. However, DOPA demonstrates slightly better diversity: in sst, while both methods retrieve strongly negative, concise opinions, DOPA’s samples vary slightly more in content and phrasing. In the implicit task, both methods capture politically charged and provocative language, but DOPA avoids redundancy by selecting stylistically consistent yet semantically distinct sentences. In contrast, KNN selects samples that, although semantically related, deviate significantly in style—favoring longer or expository sentences that mismatch the terse nature of the test examples. Overall, DOPA achieves stronger style alignment with greater diversity, while KNN struggles to capture the nuanced stylistic cues of the target domain.

\begin{figure*}[h]
  \centering
  \begin{subfigure}[t]{0.48\textwidth}
    \centering
    \includegraphics[width=\textwidth]{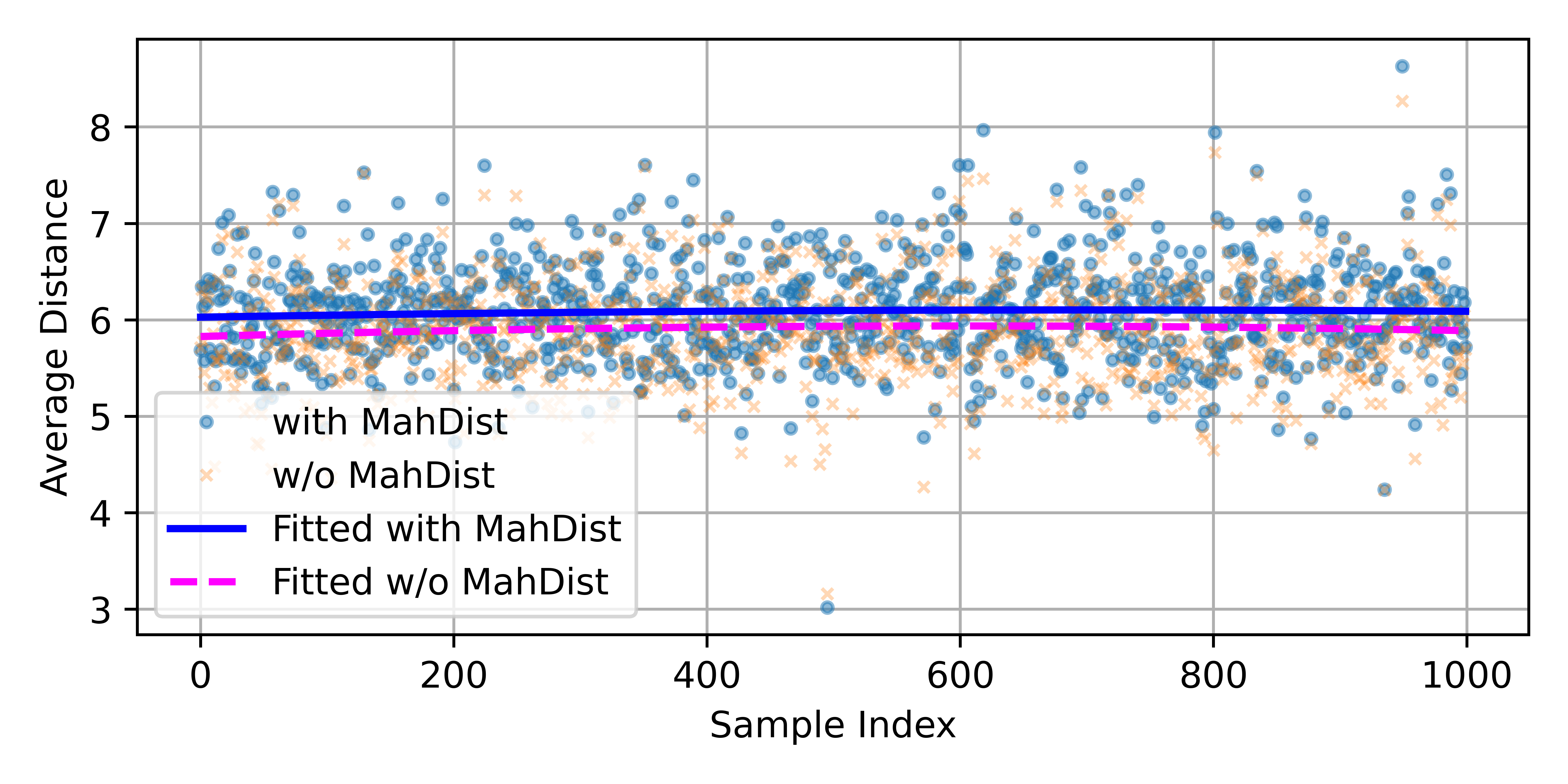}
    \caption{dynasent}
  \end{subfigure}
  \begin{subfigure}[t]{0.48\textwidth}
    \centering
    \includegraphics[width=\textwidth]{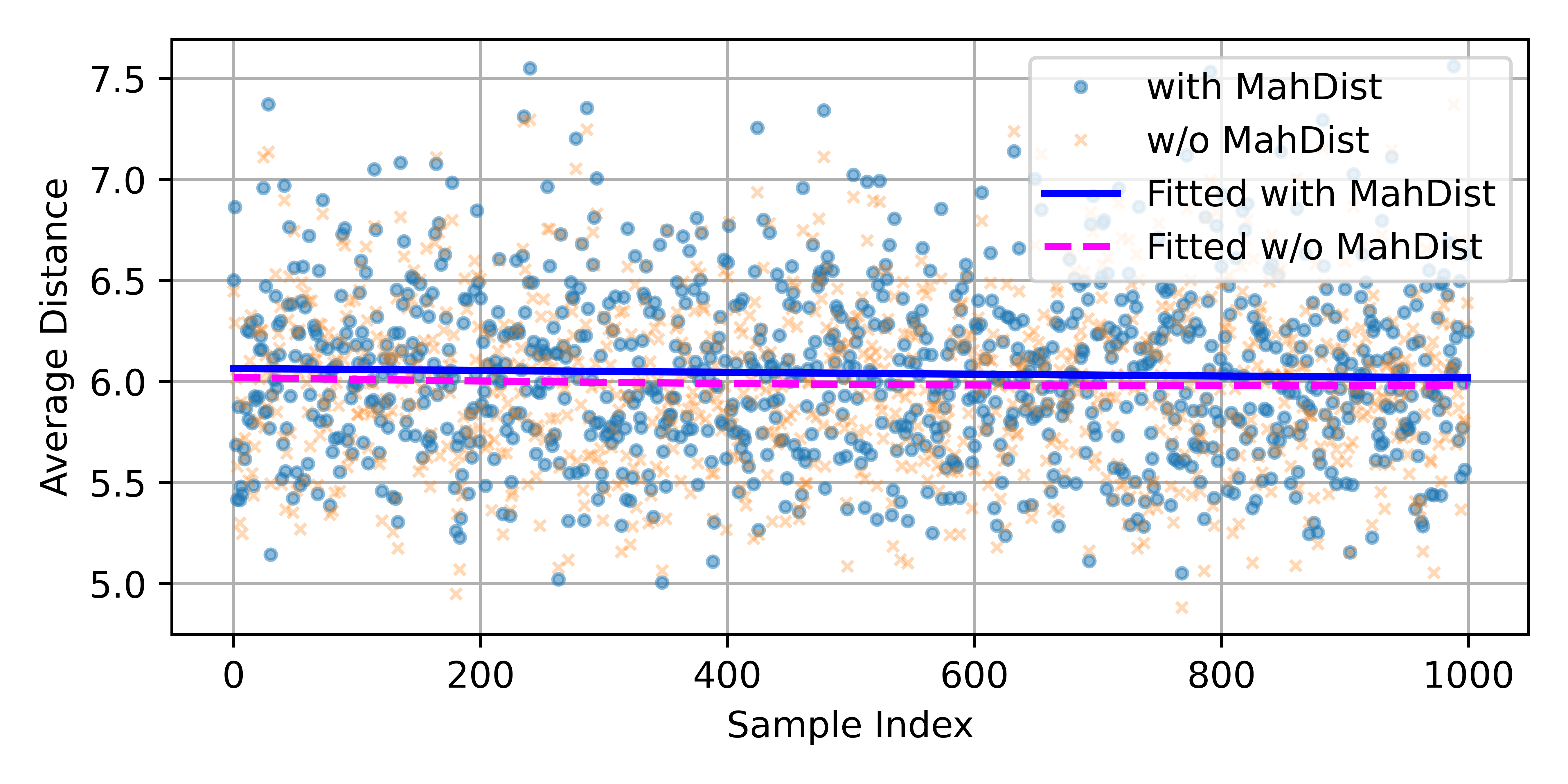}
    \caption{semeval}
  \end{subfigure}

    \begin{subfigure}[t]{0.48\textwidth}
    \centering
    \includegraphics[width=\textwidth]{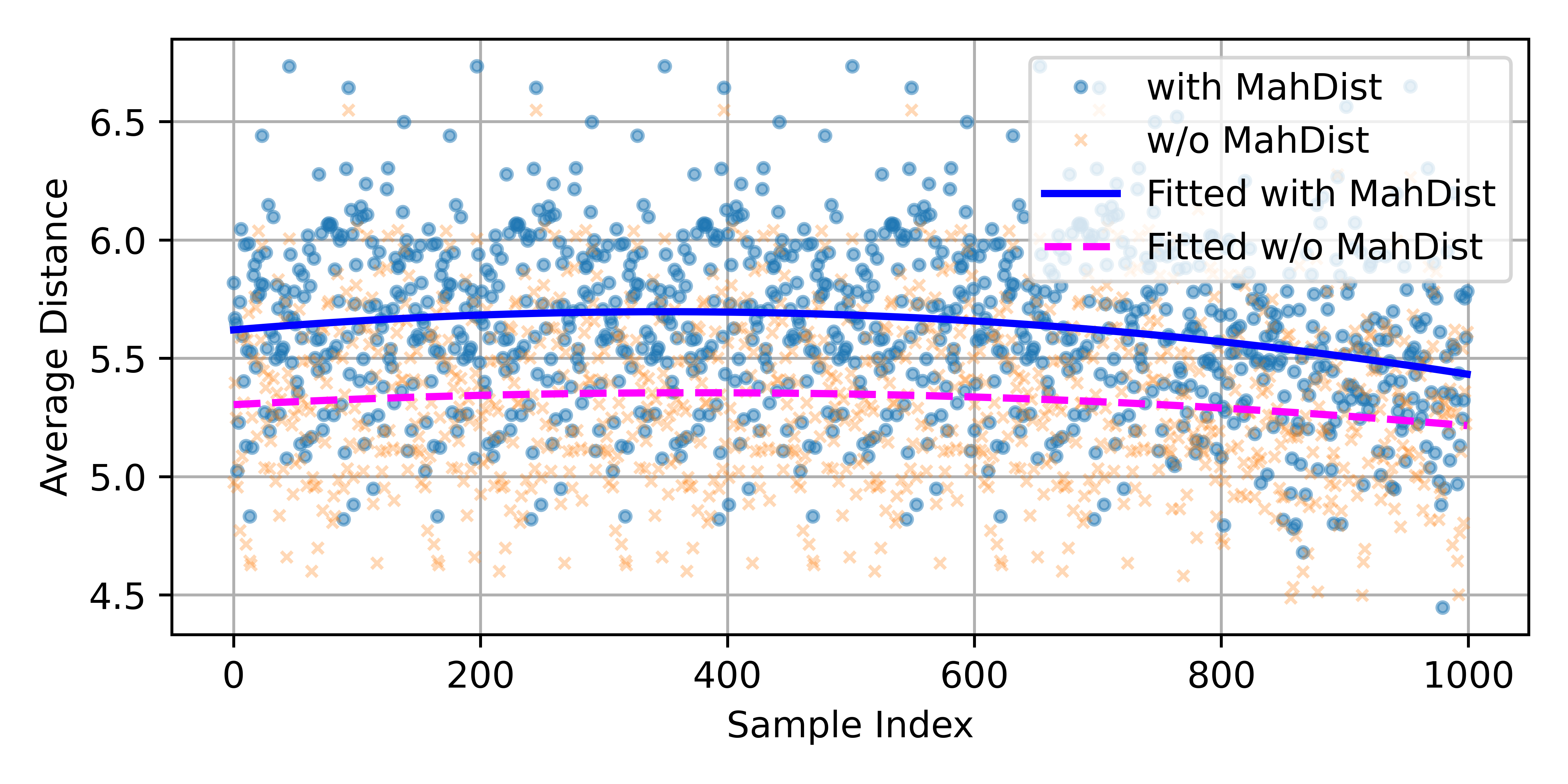}
    \caption{adv\_civil}
  \end{subfigure}
  \begin{subfigure}[t]{0.48\textwidth}
    \centering
    \includegraphics[width=\textwidth]{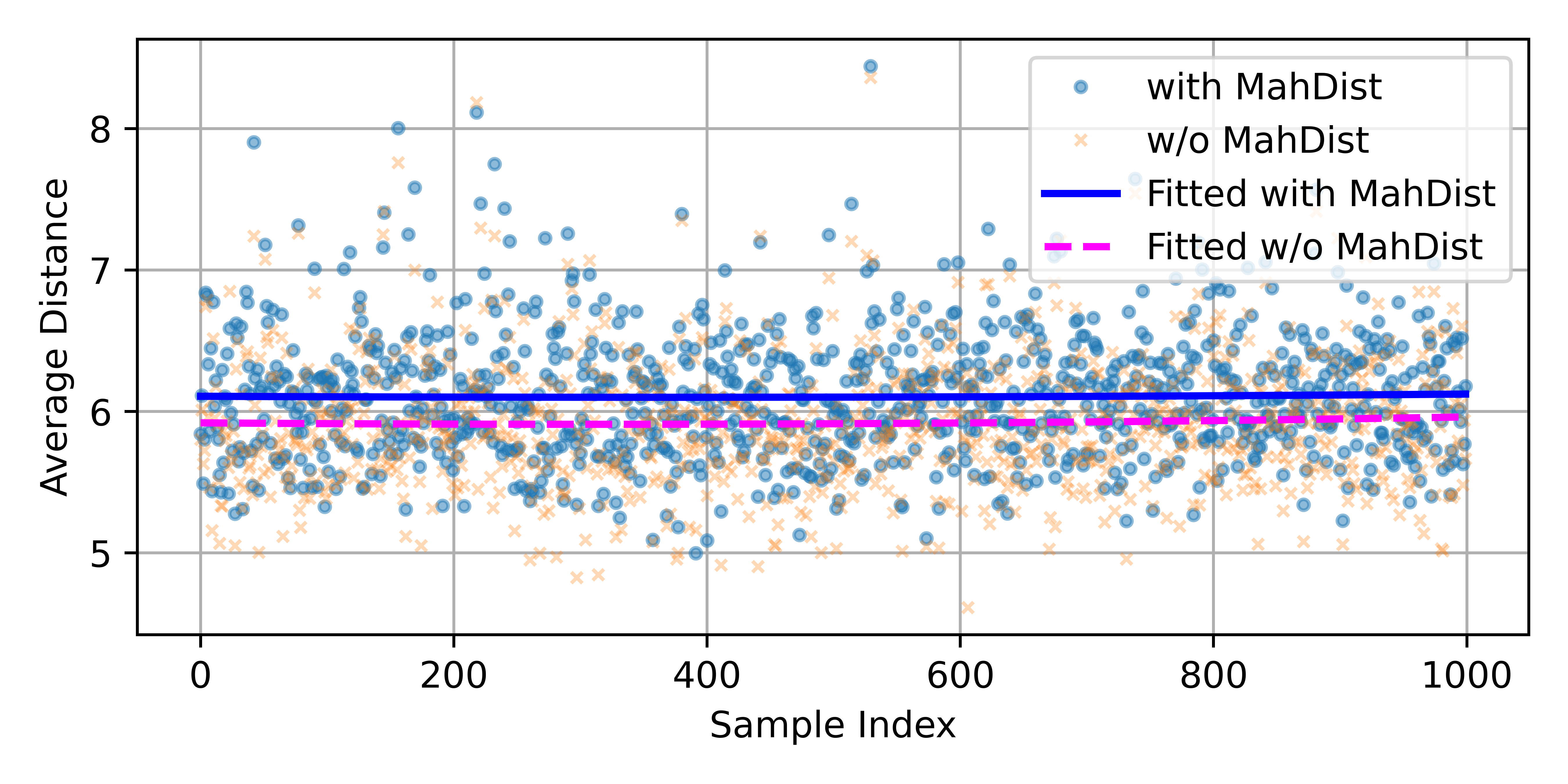}
    \caption{implicit\_hate}
  \end{subfigure}

  \begin{subfigure}[t]{0.48\textwidth}
    \centering
    \includegraphics[width=\textwidth]{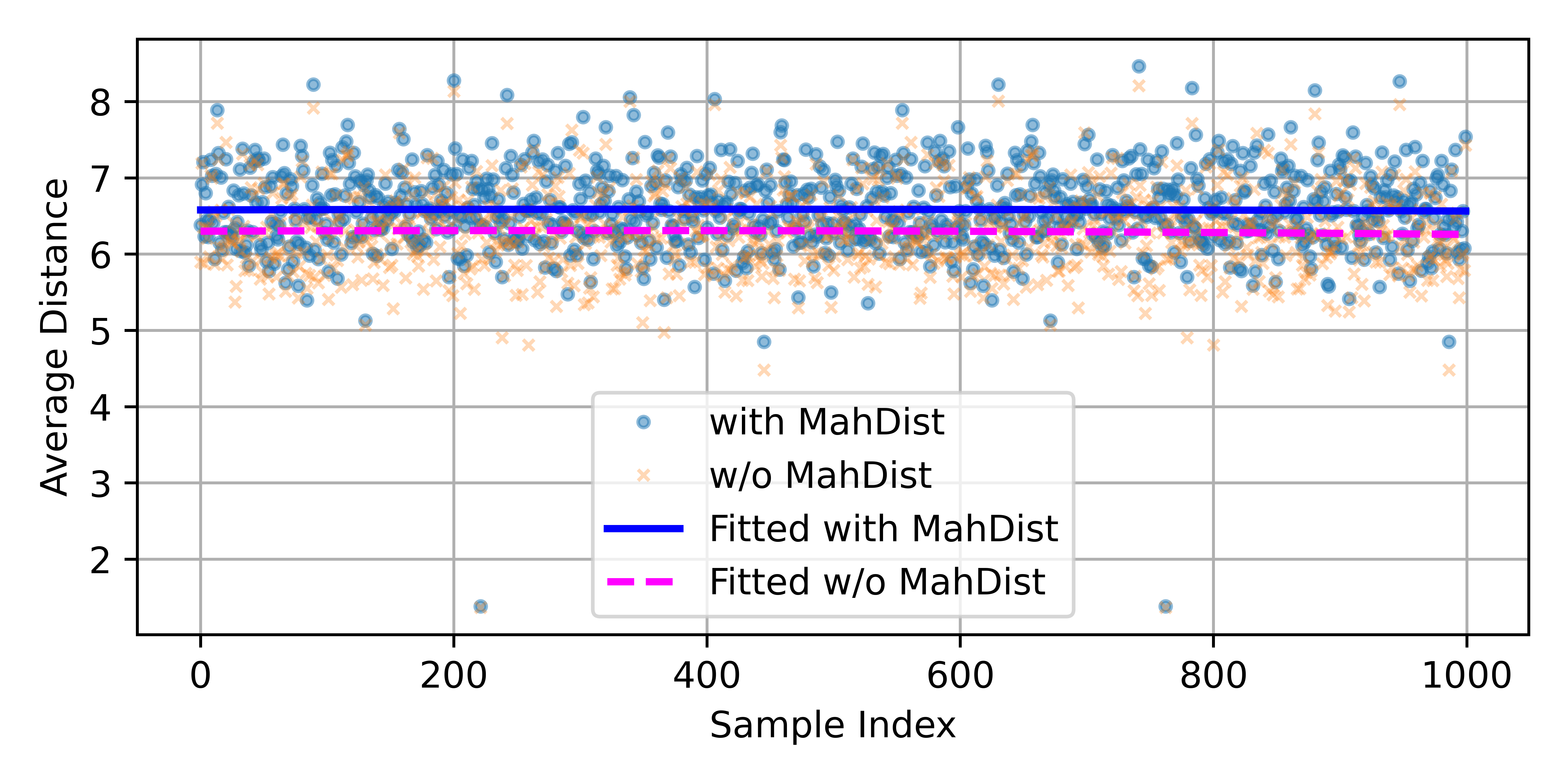}
    \caption{toxigen}
  \end{subfigure}
    \begin{subfigure}[t]{0.48\textwidth}
    \centering
    \includegraphics[width=\textwidth]{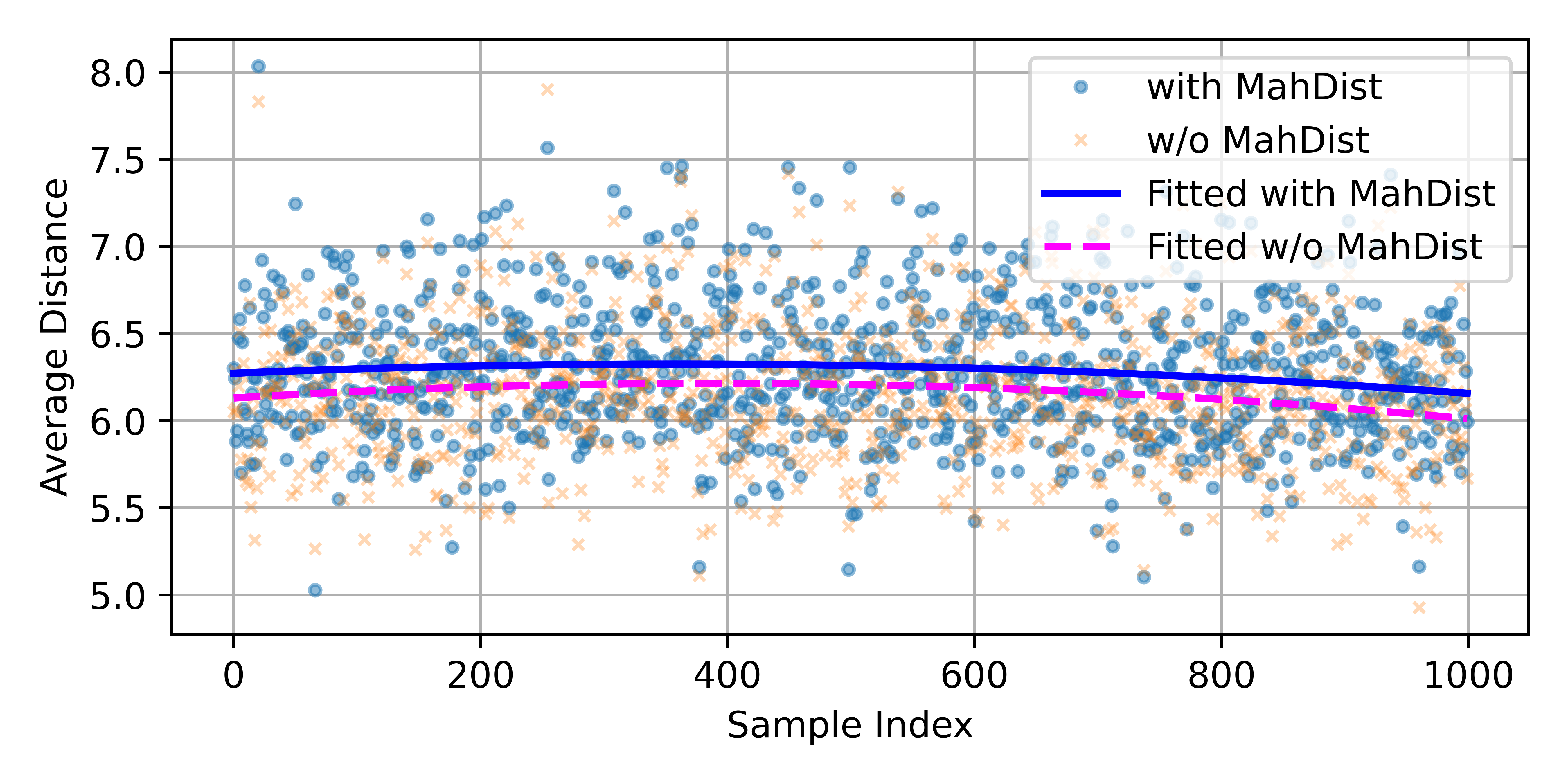}
    \caption{anli}
  \end{subfigure}

  \begin{subfigure}[t]{0.48\textwidth}
    \centering
    \includegraphics[width=\textwidth]{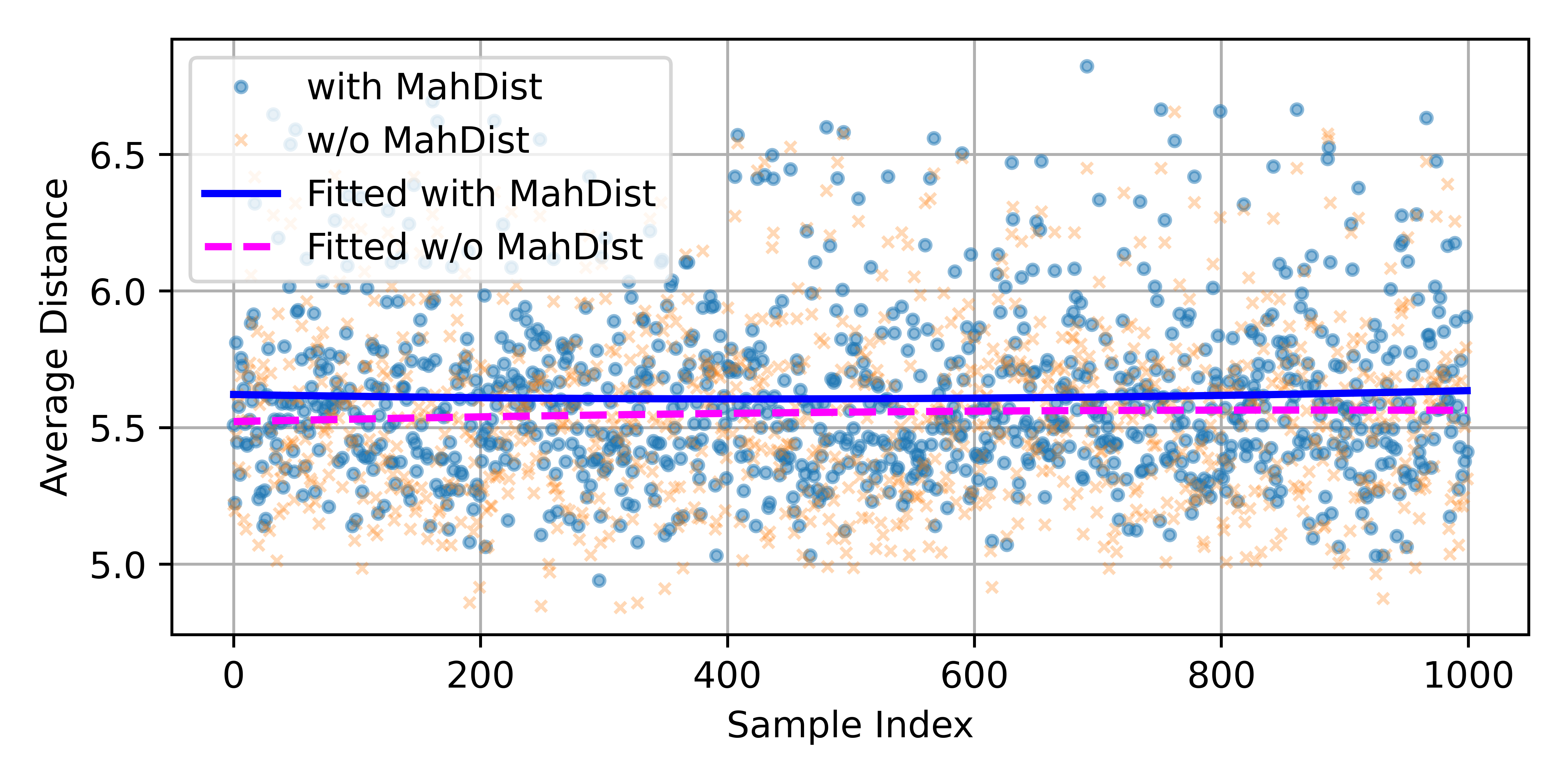}
    \caption{contract\_nli}
  \end{subfigure}
  \begin{subfigure}[t]{0.48\textwidth}
    \centering
    \includegraphics[width=\textwidth]{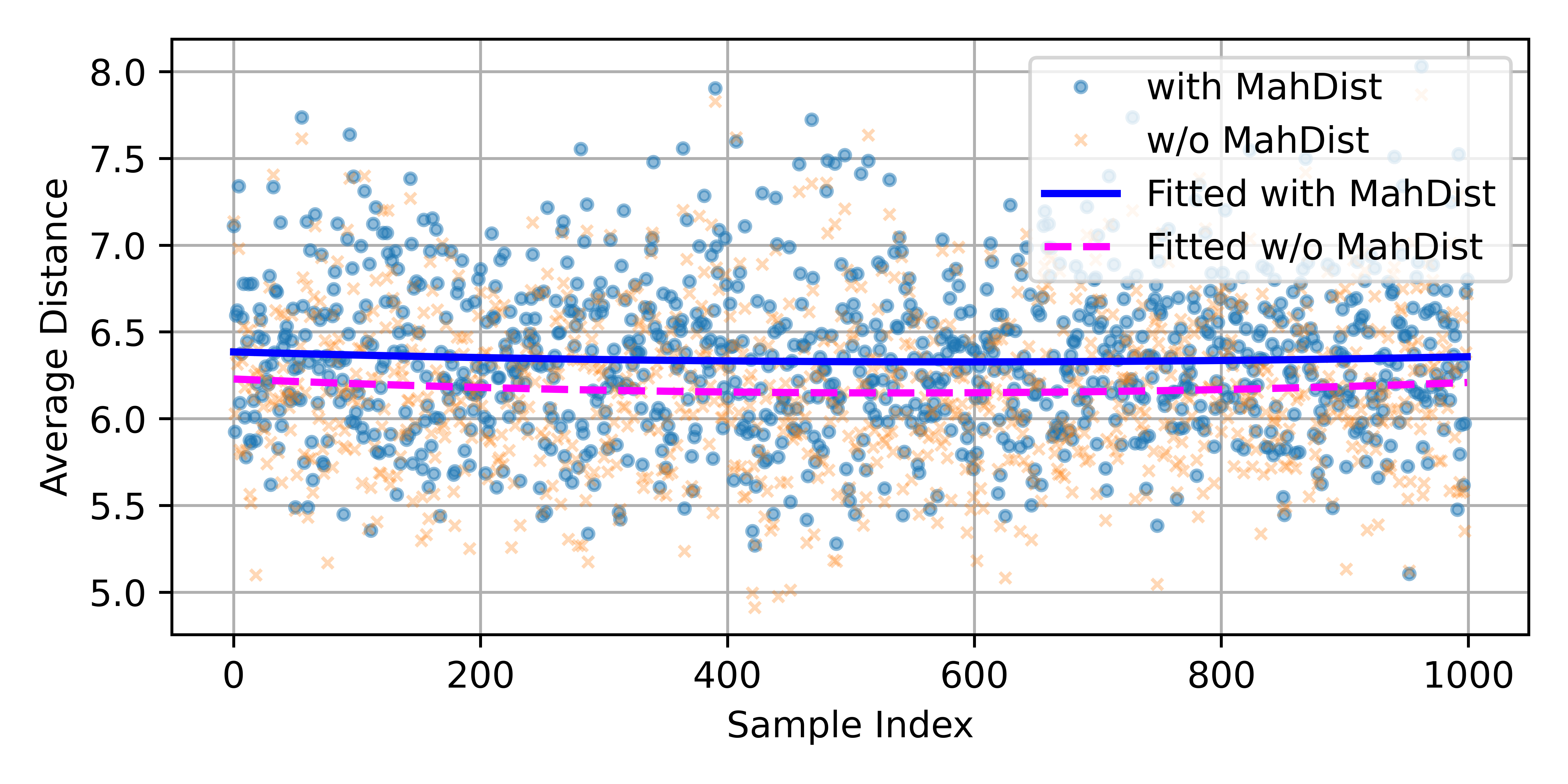}
    \caption{wanli}
  \end{subfigure}
  \caption{Euclidean distance comparison to target domain samples for retrieval results with and without the diversity constraint (with MahDist and w/o MahDist) on all tasks.}
  \label{fig:vis_all2}
\end{figure*}

\section{More Visualization Results} \label{app:vis}
We further present KDE distributions of sample representations across various tasks in Figure~\ref{fig:vis_all} to demonstrate the generality of DOPA in selecting appropriate samples. Overall, the samples selected by the proxy consistently exhibit a distribution that shifts away from the source domain and moves closer to the target domain. For example, on the \textit{implicit\_hate} dataset, the proxy-based distribution almost completely overlaps with that of the target domain. This demonstrates DOPA’s capability to effectively identify samples with similar underlying distributions to the target domain. But we also observe that in a few cases (e.g., \textit{anli}), the proxy-based distribution fails to effectively deviate from the source domain. This may be attributed to the nature of \textit{anli} itself, which is a human-crafted adversarial benchmark, making it challenging for the model to accurately capture its characteristics. We do not perform the corresponding visualization experiments on the NER dataset because it is not a sentence-level task, making it difficult to obtain the relevant probability distributions. Overall, DOPA’s ability to capture approximate distributions exhibits generalization and can perform well across most tasks.

Figure~\ref{fig:vis_all2} illustrates the effect of the diversity constraint across additional datasets. Similar to Figure~\ref{fig:sst_vis2}, the curve fitted under the with MahDist setting demonstrates greater diversity. Together with the ablation study, this provides strong evidence for the effectiveness of the diversity constraint component in DOPA.

\end{document}